\documentclass{article} 
\usepackage{iclr2025_conference,times}

\usepackage{amsmath,amsfonts,bm}

\def\eqref#1{equation~\ref{#1}}

\def\1{\bm{1}}

\DeclareMathAlphabet{\mathsfit}{\encodingdefault}{\sfdefault}{m}{sl}
\SetMathAlphabet{\mathsfit}{bold}{\encodingdefault}{\sfdefault}{bx}{n}

\newcommand{\E}{\mathbb{E}}

\DeclareMathOperator{\sign}{sign}

\usepackage[backref=page]{hyperref}
\renewcommand*{\backref}[1]{}
\renewcommand*{\backrefalt}[4]{%
    \ifcase #1 %
    \or
        (cited on page~#2)%
    \else
        (cited on pages~#2)%
    \fi
}
\usepackage{url}
\usepackage{graphicx}
\usepackage[ruled,vlined]{algorithm2e}
\SetKwInput{Input}{Input}
\usepackage{amsmath,amssymb,amsthm}
\usepackage{booktabs}
\usepackage{longtable}   

\newtheorem{theorem}{Theorem}
\newtheorem{lemma}{Lemma}
\newtheorem{remark}{Remark}
\newtheorem{assumption}{Assumption}
\usepackage{subcaption}
\usepackage{pifont}

\newcommand{\xmark}{\ding{55}} 
\usepackage{float}
\usepackage{bbm}        
\usepackage{mathtools}  

\title{Ano : Faster is Better in Noisy Landscapes \\}

\author{Adrien Kegreisz \\
Independent Researcher\\
Paris, France\\
\texttt{adrien.kegreisz@gmail.com}
}

\iclrfinalcopy 
\begin{document}

\maketitle
\begin{abstract}
Stochastic optimizers are central to deep learning, yet widely used methods such as Adam and Adan can degrade in non-stationary or noisy environments, partly due to their reliance on momentum-based magnitude estimates. We introduce Ano, a novel optimizer that decouples direction and magnitude: momentum is used for directional smoothing, while instantaneous gradient magnitudes determine step size. This design improves robustness to gradient noise while retaining the simplicity and efficiency of first-order methods. We further propose Anolog, which removes sensitivity to the momentum coefficient by expanding its window over time via a logarithmic schedule. We establish non-convex convergence guarantees with a convergence rate similar to other sign-based methods, and empirically show that Ano provides substantial gains in noisy and non-stationary regimes such as reinforcement learning, while remaining competitive on low-noise tasks.
\end{abstract}

\section{Introduction}
Stochastic optimization is central to modern deep learning. Adaptive methods such as Adam~\citep{kingma2015adam}, and their variants~\citep{reddi2018convergence,loshchilov2018decoupled,Zaheer2018Adaptive} are widely used because they automatically adjust step sizes and often accelerate early training. However, their behavior can degrade under noisy or non-stationary conditions: mini-batch stochasticity and data augmentation induce gradient noise~\citep{mandt2017stochastic}, labels may be ambiguous or noisy~\citep{zhang2017understanding,song2022learning}, and in reinforcement learning the training targets evolve over time~\citep{henderson2018matters,mnih2016asynchronous}. A key limitation is that Adam couples update direction and magnitude through momentum: prior work~\citep{balles2018dissecting} shows that the momentum sign already captures most directional information, while its magnitude and the second-moment estimate impose heavy smoothing.

We propose \textbf{A New Optimizer}, abbreviated as \textbf{Ano}, designed to handle noisy optimization landscapes. First, we decouple direction and magnitude: updates follow the momentum sign for stability, while the step size is scaled by an SNR-like ratio of instantaneous gradients, avoiding the sluggishness of momentum-based magnitudes. Second, we revisit Yogi’s asymmetric variance update \citep{Zaheer2018Adaptive}, which accelerates recovery after noise spikes, and introduce an additional decay factor to control its memory. This preserves Yogi’s fast responsiveness while ensuring smoother adaptation under highly stochastic gradients.

We summarize our main contributions as follows:
\begin{itemize}
\item We propose a new gradient-scaling mechanism that removes the reliance on momentum-based magnitude estimates, leading to better adaptation in non-stationary and/or noisy optimization landscapes with the same memory cost as Adam.
\item We provide a theoretical analysis of Ano in the non-convex setting, establishing a convergence rate of $\mathcal{O}\!\left(k^{-1/4}\right)$ under standard assumptions, matching existing results for sign-based optimizers.
\item We evaluate Ano on supervised and reinforcement learning tasks, showing clear gains in noisy, non-stationary settings while remaining competitive in standard benchmarks.
\end{itemize}

\section{Related Work}
Research on stochastic optimizers spans several directions. We briefly review the lines most relevant to Ano and situate our contribution.

\paragraph{Adaptive methods.}
AdaGrad~\citep{duchi2011adaptive}, AdaDelta~\citep{zeiler2012adadeltaadaptivelearningrate} and RMSProp~\citep{tieleman2012rmsprop} pioneered coordinate-wise adaptivity; Adam~\citep{kingma2015adam} combined first- and second-moment estimates and became a default in deep learning. Yogi~\citep{Zaheer2018Adaptive} stabilizes the second-moment accumulator for non-stationary regimes. More recently, Adan~\citep{xie2024adan} couples adaptive moments with Nesterov-style lookahead and has emerged as a competitive baseline. Our optimizer Ano relates to this family through variance-aware step-size control, but differs in how direction and magnitude are constructed.
\paragraph{Sign-based methods.}
SignSGD and Signum~\citep{bernstein2018signsgd} reduce updates to element-wise signs, offering scale invariance and communication efficiency in distributed settings. Lion~\citep{chen2023symbolic} revisits sign-based updates with a tailored momentum schedule, yielding strong empirical results. Ano keeps the robustness of sign-informed directions but reintroduces gradient magnitudes through an explicit decoupling, trading pure scale invariance for finer adaptivity.
\paragraph{Direction–magnitude decoupling.}
Recent works such as Grams~\citep{cao2024grams} decouple the update by using gradient signs for direction and the momentum norm for scaling. Ano adopts a complementary design: momentum provides a stable directional signal, while the raw gradient norm sets the step size. This hybridization aims to combine the resilience of sign-based directions with the adaptivity of moment estimators.
\paragraph{Optimization under non-stationarity.}
Non-stationarity is a known stressor for optimization, particularly in RL. Prior work tackled it with task-specific procedures such as \emph{Normalize and Project (NaP)}~\citep{lyle2024normalization} and meta-learned optimizers for RL~\citep{lan2025learning}. While these highlight the need for robustness to evolving objectives, they are not first-order per-parameter adaptive optimizers in the Adam/Lion sense. Ano instead offers a simple, general optimizer that retains such efficiency while improving stability under noise and non-stationarity.
\paragraph{Discussion.}
Ano unifies sign-based and adaptive-moment ideas via a per-parameter direction–magnitude split (momentum for direction, raw gradients for scale), which we find particularly robust in high-variance regimes while remaining competitive on standard tasks.

\section{Algorithm}
\label{sec:algorithm}
The full Ano algorithm is summarized in Algorithm \ref{alg:ano}. Like Adam, it maintains first- and second-moment estimates $m_k$, $v_k$, but introduces two key innovations described below: one targeting the decoupling of update direction and magnitude, and the other improving variance adaptation under noisy gradients.

Ano algorithm is presented below:

\begin{algorithm}[H]
\caption{Ano Algorithm}
\Input{Initial parameters $x_1 \in \mathbb{R}^d$, learning rate $\eta_k$, decay rates $\beta_1, \beta_2 \in [0,1)$, $\epsilon > 0$}
\BlankLine
Initialize $m_0 = 0$, $v_0 = 0$ \\
\For{$k = 1$ \KwTo $K$}{
Compute gradient $g_k = \nabla \ell(x_k)$ \\
$m_k = \beta_1 m_{k-1} + (1 - \beta_1) g_k$ \\
$v_k = \beta_2v_{k-1} - (1 - \beta_2) \cdot \text{sign}(v_{k-1}-g_k^2) \cdot g_k^2$ \\
$\hat{v_k} = \frac{v_k}{1-\beta2^k}$\\
$x_{k+1} = x_k - \frac{\eta_k}{\sqrt{\hat{v_k}} + \epsilon} \cdot |g_k| \cdot \text{sign}(m_k) - \eta_k \lambda x_k$\\
}
\label{alg:ano}
\end{algorithm}

\paragraph{Sign–Magnitude Decoupling.}
We explicitly decouple the direction and magnitude of parameter updates to mitigate the conservative dynamics of Adam. In Adam, both signals are derived from the momentum term $m_k$, so when large noise spikes occur, their opposing effects can partially cancel out, reducing the effective momentum and thereby slowing down the updates. Ano keeps the direction $\operatorname{sign}(m_k)$ for robustness to noise but replaces the momentum magnitude with the instantaneous gradient norm $|g_k|$ for a better scaling. 

Concretely, recall that Adam updates parameters via
\[
\begin{aligned}
x_{k+1} &= x_k - \frac{\eta_k}{\sqrt{v_k} + \epsilon} \cdot m_k = x_k - \frac{\eta_k}{\sqrt{\hat{v_k}} + \epsilon}\, 
            \underset{\text{magnitude}}{\underbrace{\bigl|m_k|}}\, 
            \cdot
            \underset{\text{direction}}{\underbrace{\operatorname{sign}(m_k)}}.
\end{aligned}
\]

Our optimiser \textbf{Ano} performs the same directional move but replaces the momentum magnitude with \( |g_k| \):
\[
x_{k+1} = x_k 
- \frac{\eta_k}{\sqrt{\hat{v_k}} + \epsilon}\,
  \underset{\text{magnitude}}{\underbrace{|g_k|}}
  \cdot
  \underset{\text{direction}}{\underbrace{\operatorname{sign}\!\bigl(m_k)}}.
\]

\paragraph{Second-Moment Term.}
Ano improves variance dynamics for stability and fast recovery in particular for non-stationnary landscape. Adam’s exponential moving average~\citep{kingma2015adam} keeps noise spikes alive for many iterations, inflating the variance estimate and shrinking steps even after the signal improves. Yogi~\citep{Zaheer2018Adaptive} addresses this with asymmetric updates for faster decay. We extend Yogi by introducing a decay factor that explicitly controls variance memory, maintaining the exponential structure while allowing smooth forgetting of outdated information. This mechanism naturally assigns greater weight to recent gradients, thereby enhancing adaptation in dynamic environments. which is essential in non-stationary environments. Formally,
\[
v_k = \beta_2 v_{k-1} - (1 - \beta_2)\, \mathrm{sign}(v_{k-1} - g_k^2)\, g_k^2,
\]
turning the variance term into a memory-controlled statistic rather than a purely reactive estimate.

\paragraph{Bias Correction and Weight Decay.}
Since Ano relies solely on the momentum direction for updates, bias correction of its magnitude is unnecessary and omitted for simplicity (same as Lion) but keep it for the variance estimate. Weight decay follows AdamW~\citep{loshchilov2018decoupled} for decoupled regularization.

\paragraph{Hyperparameters.}
Like Adam, Ano maintains first and second moments estimators : $m_k$ and $v_k$, each regulated by decay rates $\beta_1$ and $\beta_2$, with $\beta_1 \in [0, 1)$ and $\beta_2 \in [\frac{1}{2}, 1)$. We set $\beta_1 = 0.92$ and $\beta_2 = 0.99$ for stable convergence. Additionally, a weight decay coefficient $\lambda \in [0, +\infty)$ is employed to mitigate overfitting.

\section{Extension}
Inspired by our convergence analysis, we extend Ano to include a time-dependent momentum parameter $\beta_1$, resulting in a variant we denote \textbf{Anolog} (Ano with logarithmic scheduling). While Ano consistently yields the best raw performance, Anolog provides a practical advantage by removing the need to tune $\beta_1$. This reduction in hyperparameter sensitivity makes Anolog a competitive choice in scenarios with limited tuning budgets, despite its slightly lower peak performance.

We define $\beta_{1,k} = 1 - \frac{1}{\log(k+2)}$, motivated by both theoretical considerations and empirical evidence favoring slow, progressive adjustments to optimization hyperparameters. A gradually increasing $\beta_1$ enlarges the effective averaging window of the momentum, thereby reducing the impact of stochastic gradient noise as training proceeds. In contrast, more aggressive schedules (e.g., square-root) may render the momentum insufficiently responsive to recent gradient information, particularly in non-stationary settings where rapid adaptation is crucial. Section~\ref{sec:ablation} provides empirical and ablation results comparing this logarithmic schedule against square-root ($\beta_{1,k} = 1 - \frac{1}{\sqrt{k}}$) and harmonic ($\beta_{1,k} = 1 - \frac{1}{k}$) schedules.

 Full Anolog pseudo code can be found in Appendix \ref{app:anolog-pseudocode} - Algorithm \ref{alg:anolog}.
 
\section{Analysis}

\subsection{Theoretical Analysis}
\label{sec:theory-analysis}

We provide non-asymptotic convergence guarantees for \textbf{Ano} under standard assumptions commonly used in adaptive stochastic optimisation~\citep{kingma2015adam,reddi2018convergence}. Consider the stochastic optimisation problem:
\(
  \min_{x \in \mathbb{R}^{d}} f(x),
\)
where \(f\) is differentiable, \(L\)-smooth, and bounded below. Let \(g_{k,i}\) denote the \(i\)-th coordinate of the stochastic gradient at iteration \(k\) and \( \mathcal{F}_{k-1}\) the filtration k-1. We assume that the gradient is bounded, \(|\nabla_i f(x_k)| \le G, \forall x \in \mathbb{R}^d\), the stochastic gradient is unbiased \((\mathbb{E}[g_{k,i}\mid \mathcal{F}_{k-1}] = \nabla_i f(x_k))\), and the variance is bounded \((\mathbb{E}[(g_{k,i}-\nabla_i f(x_k))^{2}\mid \mathcal{F}_{k-1}] \le \sigma^{2})\).

\paragraph{Main result.}
Following recent convergence analyses of sign-based optimizers, especially Lion \citep{dong2024convergence}, we assume a learning-rate schedule \(\eta_k = \eta/k^{3/4}\) and \(\beta_{1,k} = 1 - 1/\sqrt{k}\), the iterates generated by \textbf{Ano} satisfy:
\[
  \min_{0 \le k < K} \mathbb{E}[\|\nabla f(x_k)\|^{2}]
  \;=\; \mathcal{O}(K^{-1/4}\log K) \;=\; \mathcal{\tilde{O}}(K^{-1/4}),
\]
up to logarithmic factors, in the general non-convex stochastic setting.

\paragraph{Proof sketch.}
Using a sign-mismatch lemma (Lemma~\ref{lem:sign-lemma}, Appendix~\ref{app:convergence_proof}), we show that the probability of momentum--gradient disagreement decays as $\mathcal{O}(1/\sqrt{k})$.
Then, using \(L\)-smoothness and the previous lemma, we establish the inequality:
\[
  \mathbb{E}[f(x_{k+1})]
  \;\le\;
  \mathbb{E}[f(x_k)]
  \;-
  \frac{\eta_k}{\tilde{G}+\varepsilon}\,
  \mathbb{E}[\|\nabla f(x_k)\|^{2}]
  \;+
  \mathcal{O}\Bigl(\frac{\eta_k}{k^{1/4}}\Bigr)
  \;+
  \mathcal{O}(\eta_k^{2}),
\]
where the last two terms represent, respectively, stochastic noise and the adaptivity of the step size. 
\paragraph{Discussion.}
Our bound matches those recently established for sign-based optimizers such as \textit{Lion}\citep{dong2024convergence} and \textit{Signum}\citep{bernstein2018signsgd}, while relying on less restrictive assumptions (e.g., no requirement for growing batch sizes). Compared to adaptive schemes (SGD, Adam, Yogi) achieving $\mathcal{O}(K^{-1/2})$, our $\tilde{\mathcal{O}}(K^{-1/4})$ rate stems from a fundamental limitation of sign-based methods: ensuring stable updates requires decaying step sizes $\eta_k = \mathcal{O}(k^{-3/4})$ which, in turn, constrains the overall convergence rate. Full proofs are in Appendix~\ref{app:convergence_proof}.

\subsection{Noise Robustness Analysis}
\label{sec:noise-robustness}

We assess noise robustness by training a CIFAR-10 CNN adding Gaussian noise  $g_k \leftarrow g_k + \mathcal N(0,\sigma^2)$ into every mini-batch gradient before the optimizer update \citep{krizhevsky2009learning}. We vary only the noise level $\sigma$ over five values, keep each optimizer's default $\beta$ and recommended learning rate (Full hyperparameters tabs can be found in Appendix \ref{app:tuning_experiments}) for a computer vision task and report mean test accuracy over 5 seeds.\footnote{95\%CI omitted here for readability. The full table with 95\% CI is available in Appendix \ref{app:additionnal Results} - tab \ref{tab:noise-robustness-detailed}}

\begin{table}[h]
\centering
\begin{tabular}{lccccc}
\toprule
\textbf{Optimizer} & $\sigma\!=\!0$ & 0.01 & 0.05 & 0.10 & 0.20 \\
\midrule
Ano   & 82.10 & 78.71 & 70.88 & 65.93 & 59.54 \\
Adam  & 80.67\textbf{(-1.43)} & 75.97\textbf{(-2.74)} & 66.86\textbf{(-4.02)} & 60.83\textbf{(-5.10)} & 52.46\textbf{(-7.08)} \\
Lion  & 81.04\textbf{(-1.05)} & 77.80\textbf{(-0.91)} & 69.62\textbf{(-1.26)} & 64.02\textbf{(-1.91)} & 56.82\textbf{(-2.72)} \\
Grams & 71.34\textbf{(-10.76)} & 77.90\textbf{(-0.81)} & 70.57\textbf{(-0.31)} & 65.47\textbf{(-0.46)} & 58.80\textbf{(-0.74)} \\
\bottomrule
\end{tabular}
\caption{CIFAR-10 test accuracy (\%). Numbers in parentheses indicate the gap (percentage points) relative to \textit{Ano}.}
\label{tab:noise-robustness}
\end{table}

The performance gap between Ano/Adam and Ano/Lion widens with noise magnitude, reaching a $-6.8$-point advantage at $\sigma\!=\!0.20$ for Adam and a $-2.7$-point advantage for Lion (Table~\ref{tab:noise-robustness}). Another noteworthy observation is that Grams improves with a small injected noise ($\sigma = 0.01$). We hypothesize that this injected perturbation amplifies short‑term oscillations, enlarging its second‑moment (variance) estimate and thereby shrinking the step size, allowing Grams to refine its iterates more cautiously in a noisy landscape. Overall, these results support our central claim that decoupling update direction from magnitude stabilizes learning under high variance, avoiding the over-smoothing of momentum-coupled schemes.

\section{Experiments}
\label{sec:experiments}
We benchmark \textbf{Ano} and its extension \textbf{Anolog} against established optimizers and similar optimizers across computer vision, natural language processing, and deep reinforcement learning. Hyperparameters for all methods are selected via per-domain (DRL, CV, NLP) proxy grid searches: each optimizer receives a fixed 40 GPU-hour budget centered on litterature defaults(cf Appendix \ref{app:tuning_experiments}). Final results are averaged over 5 seeds for CV/NLP and 10 for DRL. Full search spaces, selected configurations, logs and codes are released (Appendices \ref{app:tuning_experiments}-\ref{app:all_hyperparameters_settings}).

\subsection{Computer Vision}
Computer vision is a historically important domain for benchmarking optimization algorithms. In this section, we evaluate Ano on CIFAR-100 \citep{krizhevsky2009learning} using ResNet-34 \citep{he2016deep}. 

\paragraph{CIFAR-100.} We use the standard CIFAR augmentation (random crop with 4-pixel padding + horizontal flip), following \cite{zagoruyko2017wideresidualnetworks}.

\begin{figure}[h]
\centering
\begin{minipage}{0.48\textwidth}
    \centering
    \includegraphics[width=\textwidth]{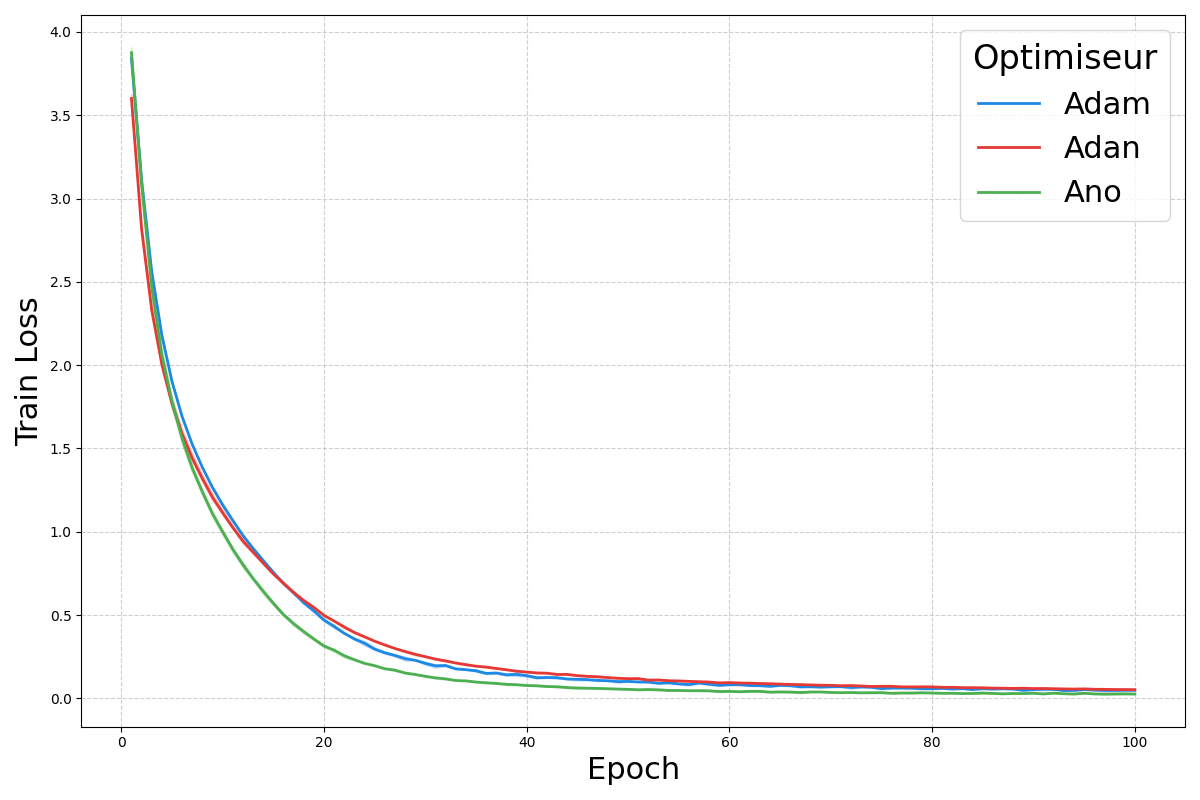}
    \caption{Training loss on CIFAR-100. Ano reduces loss faster and more stably than Adam.}
    \label{fig:cifar-loss}
\end{minipage}
\hfill
\begin{minipage}{0.35\textwidth}
\centering
\resizebox{\textwidth}{!}{%
\begin{tabular}{lcc}
    \toprule
    Optimizers & Test Accuracy & Training Loss  \\
    \midrule
    \textit{Default} \\
    Adam       & $69.57 \pm 0.22$ & $0.037$ \\
    Adan    & $69.87 \pm 0.09$ & $0.049$ \\
    Lion      & $66.25 \pm 0.61$ & $0.064$ \\
    Grams     & $68.33 \pm 0.40$ & $0.045$ \\
    Ano       & $\mathbf{70.31} \pm \mathbf{0.50}$ & $\mathbf{0.015}$ \\
    Anolog       & $64.84 \pm 1.19$ & $0.019$ \\
    \addlinespace
    \textit{Tuned} \\
    Adam      & $69.61 \pm 0.23$ & $0.042$ \\
    Adan      & $69.09 \pm 0.16$ & $0.049$ \\
    Lion      & $68.77 \pm 0.26$ & $0.048$ \\
    Grams     & $68.11 \pm 0.27$ & $0.048$ \\
    Ano       & $\mathbf{69.89} \pm \mathbf{0.42}$ & $\mathbf{0.022}$ \\
    Anolog       & $68.41 \pm 0.58$ & $\mathbf{0.032}$ \\
    \bottomrule
\end{tabular}
}
\captionof{table}{Test accuracy and training loss on CIFAR-100.}
\label{tab:cifar-results}
\end{minipage}
\end{figure}

As shown in Table~\ref{tab:cifar-results}, \textbf{Ano}  outperforms Adam and Adan in both default and tuned settings. Figure~\ref{fig:cifar-loss} further illustrates that Ano achieves faster convergence and lower training loss throughout training.



\subsection{Natural Language Processing}

As a cornerstone of modern artificial intelligence, natural language processing (NLP) warrants careful evaluation of optimization algorithms. The unique challenges of NLP, such as large parameter spaces, sparse gradients, and label noise, can affect optimizer performance~\citep{mosbach2021on}. We assess the effectiveness of \textbf{Ano} by comparing it to standard baselines on the GLUE benchmark~\citep{wang2019gluemultitaskbenchmarkanalysis}, covering eight sentence- and sentence-pair classification tasks (excluding WNLI, following standard practice due to its unreliable performance).

\paragraph{GLUE}
All runs finetune the public \texttt{bert-base-uncased} checkpoint~\citep{devlin2019bert} with max sequence length 128, batch size 32, weight decay 0.01, linear schedule with 10\% warmup, mixed precision, and 3 epochs (5 for small datasets: CoLA, MRPC, RTE, STS-B). Each configuration is repeated over 5 seeds.

\begin{table}[h]
\centering
\resizebox{\textwidth}{!}{
\begin{tabular}{lcccccccc|c}
\toprule
Optimizer & CoLA & MNLI & MRPC & QNLI & QQP & RTE & SST-2 & STS-B & Average \\
\midrule
\textit{Default} \\
Adam     & $\mathbf{59.40} \pm \mathbf{1.67}$ & $\mathbf{84.62} \pm \mathbf{0.10}$ & $88.06 \pm 0.82$ & $\mathbf{91.60} \pm \mathbf{0.15}$ & $89.64 \pm 0.10$ & $66.67 \pm 1.59$ & $92.73 \pm 0.46$ & $88.44 \pm 0.27$ & $82.64$ \\
Adan      & $55.65 \pm 0.53$ & $84.17 \pm 0.07$ & $84.40 \pm 0.83$ & $91.10 \pm 0.14$ & $88.85 \pm 0.04$ & $61.49 \pm 1.30$ & $92.02 \pm 0.20$ & $87.26 \pm 0.63$ & $80.62$ \\
Lion      & $57.76 \pm 1.76$ & $83.76 \pm 0.23$ & $87.13 \pm 0.81$ & $90.63 \pm 0.68$ & $89.46 \pm 0.05$ & $62.89 \pm 1.17$ & $91.67 \pm 0.52$ & $88.00 \pm 0.25$ & $81.41$ \\
Grams     & $56.15 \pm 0.92$ & $83.89 \pm 0.11$ & $84.92 \pm 0.60$ & $91.10 \pm 0.04$ & $88.48 \pm 0.08$ & $63.36 \pm 1.60$ & $92.34 \pm 0.19$ & $87.57 \pm 0.32$ & $80.98$ \\
\textbf{Ano (Ours)}     & $58.36 \pm 1.15$ & $84.33 \pm 0.17$ & $\mathbf{88.96} \pm \mathbf{0.50}$ & $91.25 \pm 0.46$ & $\mathbf{89.71} \pm \mathbf{0.11}$ & $\mathbf{69.25} \pm \mathbf{2.94}$ & $\mathbf{92.80} \pm \mathbf{0.41}$ & $88.70 \pm 0.12$ & $\mathbf{82.92}$ \\
\textbf{Anolog (Ours)}  & $57.07 \pm 2.41$ & $84.55 \pm 0.09$ & $88.26 \pm 0.76$ & $91.51 \pm 0.10$ & $\mathbf{89.71} \pm \mathbf{0.12}$ & $67.87 \pm 1.94$ & $92.75 \pm 0.15$ & $\mathbf{88.95} \pm \mathbf{0.32}$ & $82.58$ \\
\addlinespace
\textit{Tuned} \\
Adam     & $57.66 \pm 2.39$ & $84.18 \pm 0.16$ & $88.09 \pm 0.79$ & $91.12 \pm 0.17$ & $89.55 \pm 0.08$ & $68.47 \pm 2.91$ & $92.18 \pm 0.08$ & $88.76 \pm 0.36$ & $82.50$ \\
Adan      & $57.71 \pm 0.92$ & $\mathbf{84.84} \pm \mathbf{0.10}$ & $88.14 \pm 0.40$ & $91.71 \pm 0.21$ & $\mathbf{89.78} \pm \mathbf{0.07}$ & $65.40 \pm 1.66$ & $92.68 \pm 0.26$ & $88.57 \pm 0.44$ & $82.35$ \\
Lion      & $56.30 \pm 0.55$ & $82.38 \pm 0.06$ & $86.83 \pm 2.91$ & $90.36 \pm 0.42$ & $88.60 \pm 0.13$ & $63.75 \pm 5.50$ & $91.47 \pm 0.24$ & $88.58 \pm 0.40$ & $81.03$ \\
Grams     & $58.18 \pm 1.12$ & $84.64 \pm 0.11$ & $\mathbf{89.05} \pm \mathbf{0.36}$ & $\mathbf{91.79} \pm \mathbf{0.17}$ & $89.66 \pm 0.06$ & $67.22 \pm 2.55$ & $\mathbf{92.98} \pm \mathbf{0.31}$ & $88.53 \pm 0.26$ & $82.76$ \\
\textbf{Ano (Ours)}     & $\mathbf{58.51} \pm \mathbf{0.75}$ & $84.39 \pm 0.12$ & $88.53 \pm 1.14$ & $91.30 \pm 0.48$ & $89.73 \pm 0.07$ & $\mathbf{69.25} \pm \mathbf{3.01}$ & $92.66 \pm 0.14$ & $88.74 \pm 0.11$ & $\mathbf{82.89}$ \\
\textbf{Anolog (Ours)}  & $57.07 \pm 2.41$ & $84.55 \pm 0.09$ & $88.26 \pm 0.76$ & $91.51 \pm 0.10$ & $89.71 \pm 0.12$ & $67.87 \pm 1.94$ & $92.75 \pm 0.15$ & $\mathbf{88.95} \pm \mathbf{0.32}$ & $82.58$ \\
\bottomrule
\end{tabular}
}
\caption{Average performance (mean $\pm$ CI95\%) of different optimizers on GLUE benchmark tasks (best per column in bold).}
\label{tab:nlp-results}
\end{table}

As shown in Table~\ref{tab:nlp-results}, \textbf{Ano} and its logarithmic version (\textbf{Anolog}) achieve the highest average score across GLUE. This improvement is mainly driven by its performance on small-scale and inherently noisy tasks such as MRPC, CoLA and RTE, which are known to exhibit high gradient variance. These results suggest that Ano's advantages are most pronounced on noisy, low-resource tasks as we can expect.

\subsection{Deep Reinforcement learning}

Reinforcement learning (RL) presents unique challenges, such as high gradient variance and the non-stationarity of the environment, both of which can strongly affect the behavior of optimization algorithms \citep{henderson2018matters, franccois2018introduction}. Because RL is inherently noisy and largely non-stationary, it is precisely the setting where Ano is expected to provide the largest benefits. For computational efficiency, we perform grid searches on \textit{HalfCheetah} using 100k training steps. We acknowledge that this shorter horizon may bias tuning toward larger learning rates, which can be suboptimal for longer runs (e.g., 1M steps). To mitigate this, for each baseline optimizer we report the best performance between its default and tuned hyperparameters, ensuring that no method is disadvantaged by the tuning protocol. The tables indicate which configuration (default or tuned) was used, with full results provided in \ref{app:additionnal Results}.

\paragraph{Soft-Actor Critic.}
In this section, we employ the Soft Actor-Critic (SAC) algorithm \citep{haarnoja2018soft} in the MuJoCo suite from the Gymnasium framework \citep{todorov2012mujoco, towers2024gymnasiumstandardinterfacereinforcement}. We reuse the standard SAC hyperparameter (full list in Appendix \ref{app:all_hyperparameters_settings} - Tab \ref{tab:sac_hparams_main}), as reported in the original work and subsequent studies and only vary the optimizer for actor, critics, and temperature. For each optimizer, we run 10 seeds, with 1M steps. We report below the average mean score on a 50-episodes test evaluation.

\begin{figure}[h]
    \centering
    \begin{subfigure}[b]{0.30\textwidth}
        \centering
        \includegraphics[width=\textwidth]{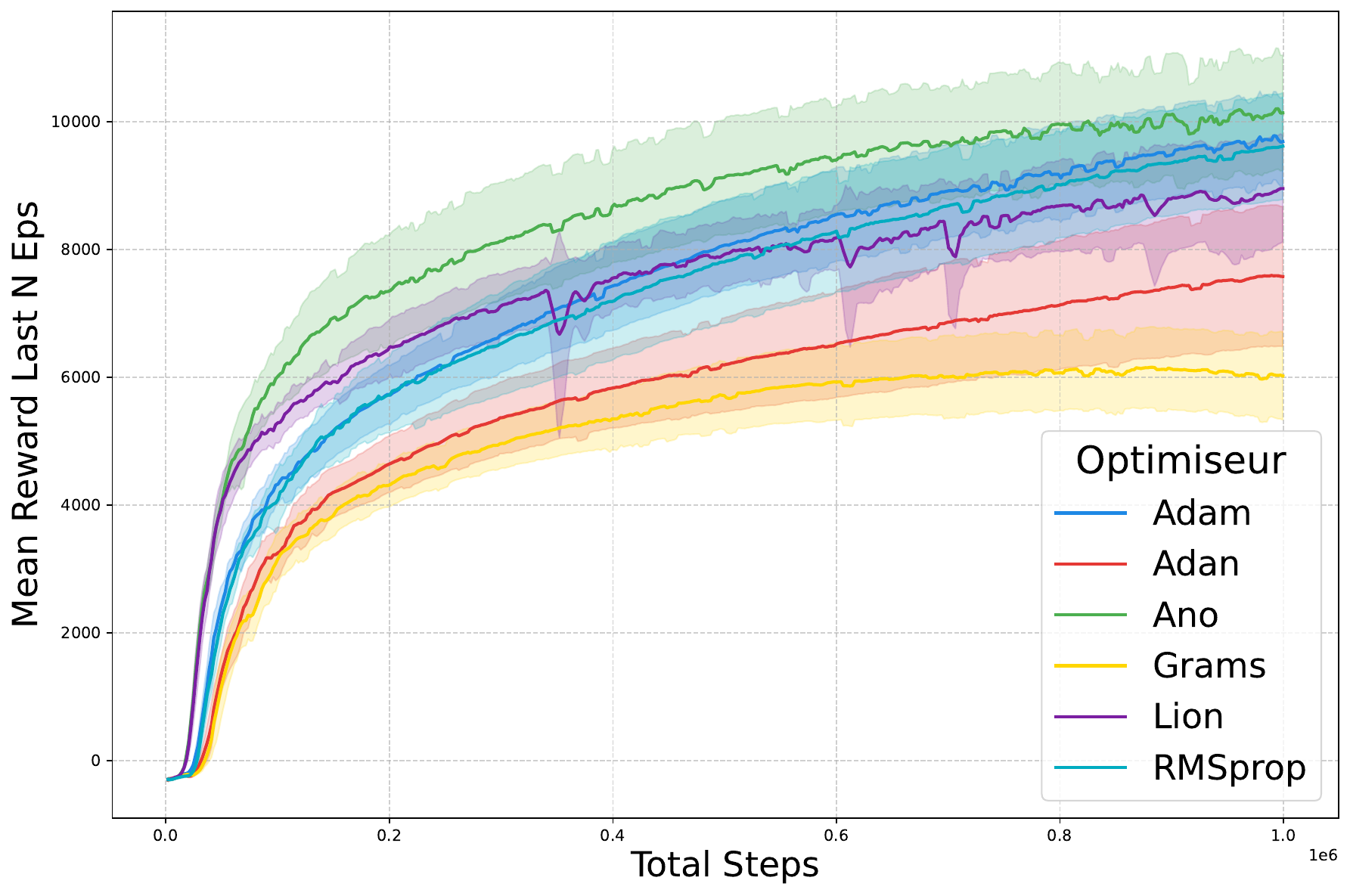}
        \caption{Halfcheetah-v5}
        \label{fig:halfcheetah-plot}
    \end{subfigure}
    \hfill
    \begin{subfigure}[b]{0.30\textwidth}
        \centering
        \includegraphics[width=\textwidth]{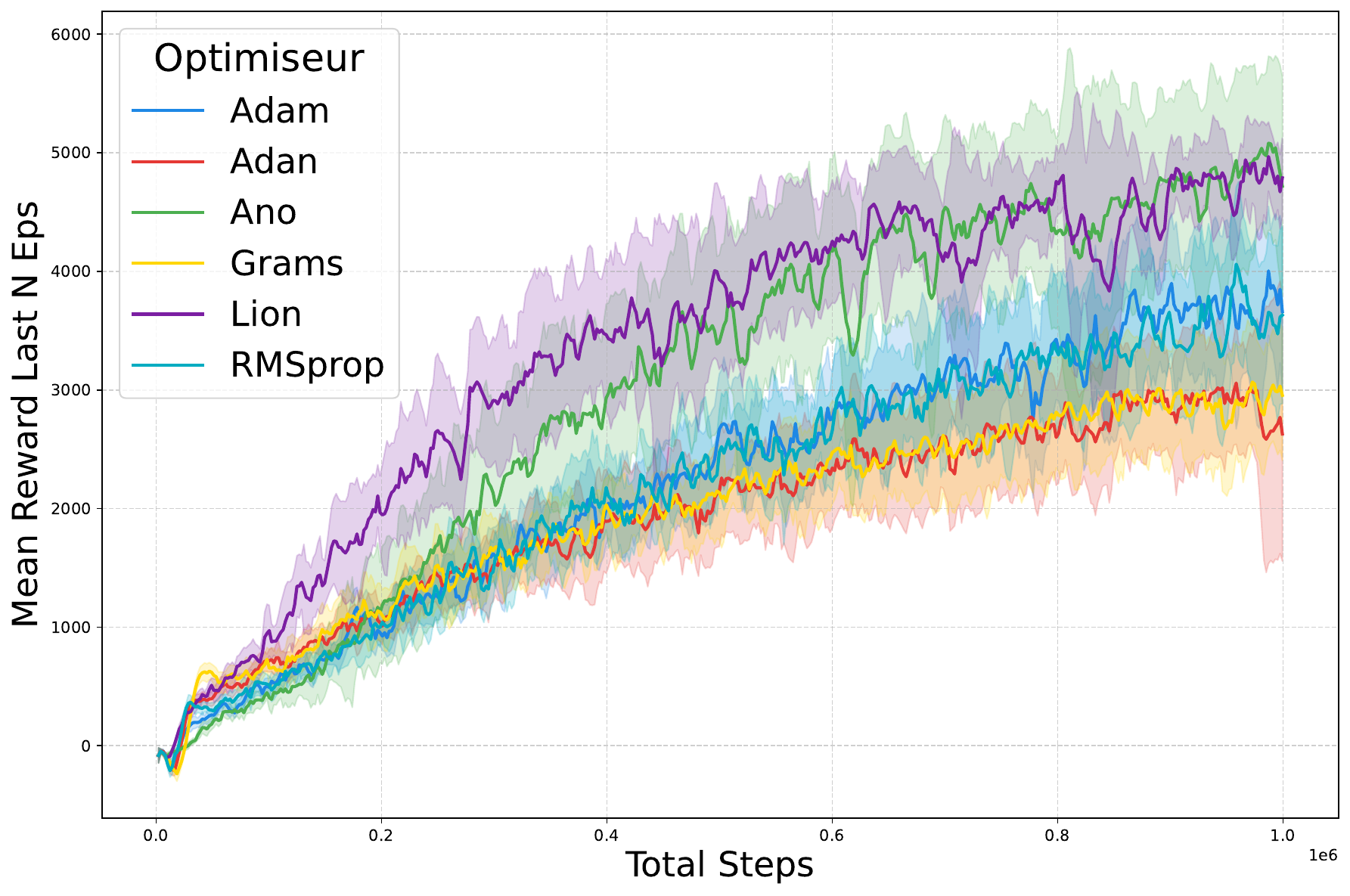}
        \caption{Ant-v5}
        \label{fig:ant-plot}
    \end{subfigure}
    \hfill
    \begin{subfigure}[b]{0.30\textwidth}
        \centering
        \includegraphics[width=\textwidth]{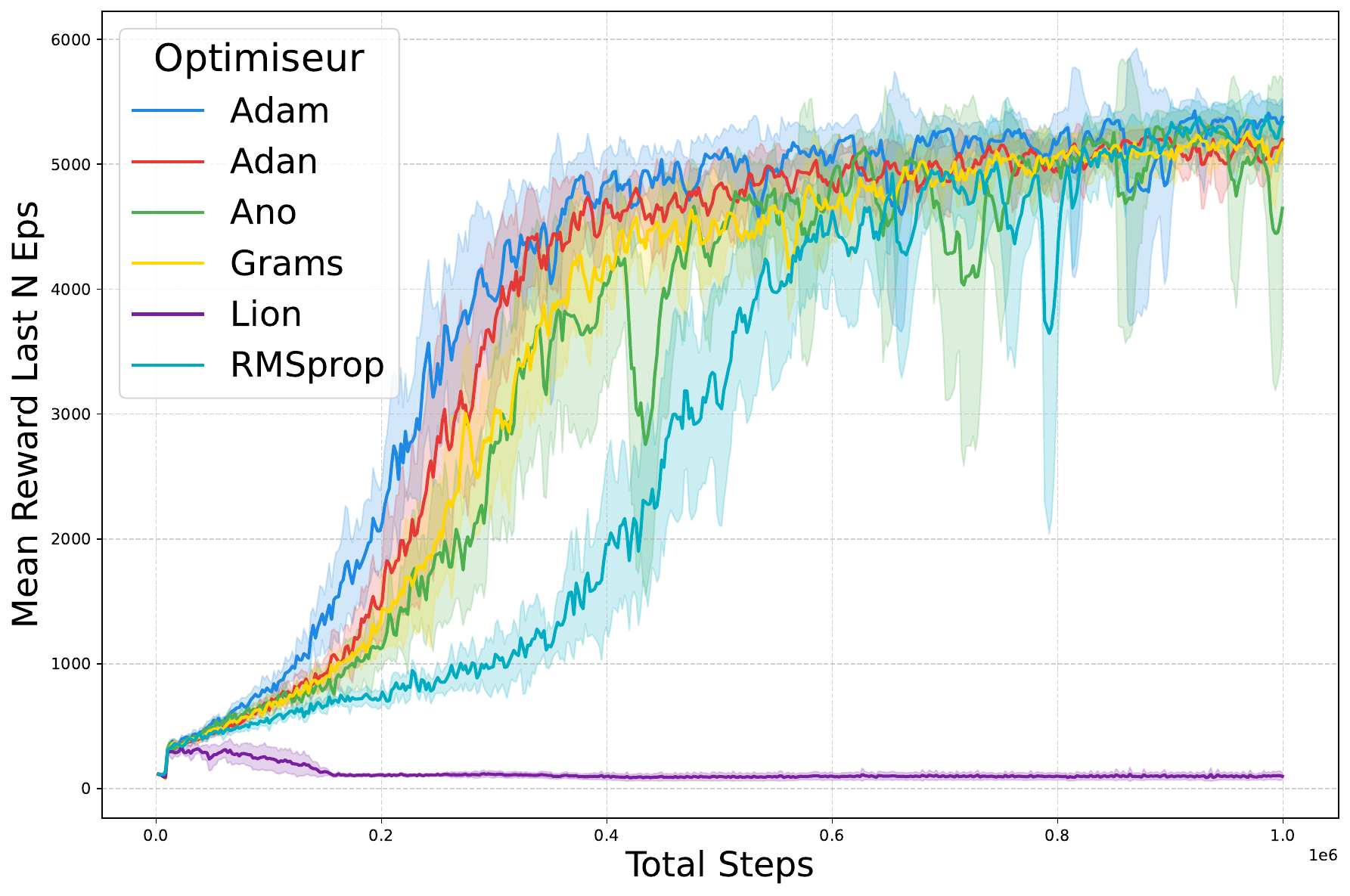}
        \caption{Humanoid-v5}
        \label{fig:humanoid-plot}
    \end{subfigure}
    \hfill
    \begin{subfigure}[b]{0.30\textwidth}
        \centering
        \includegraphics[width=\textwidth]{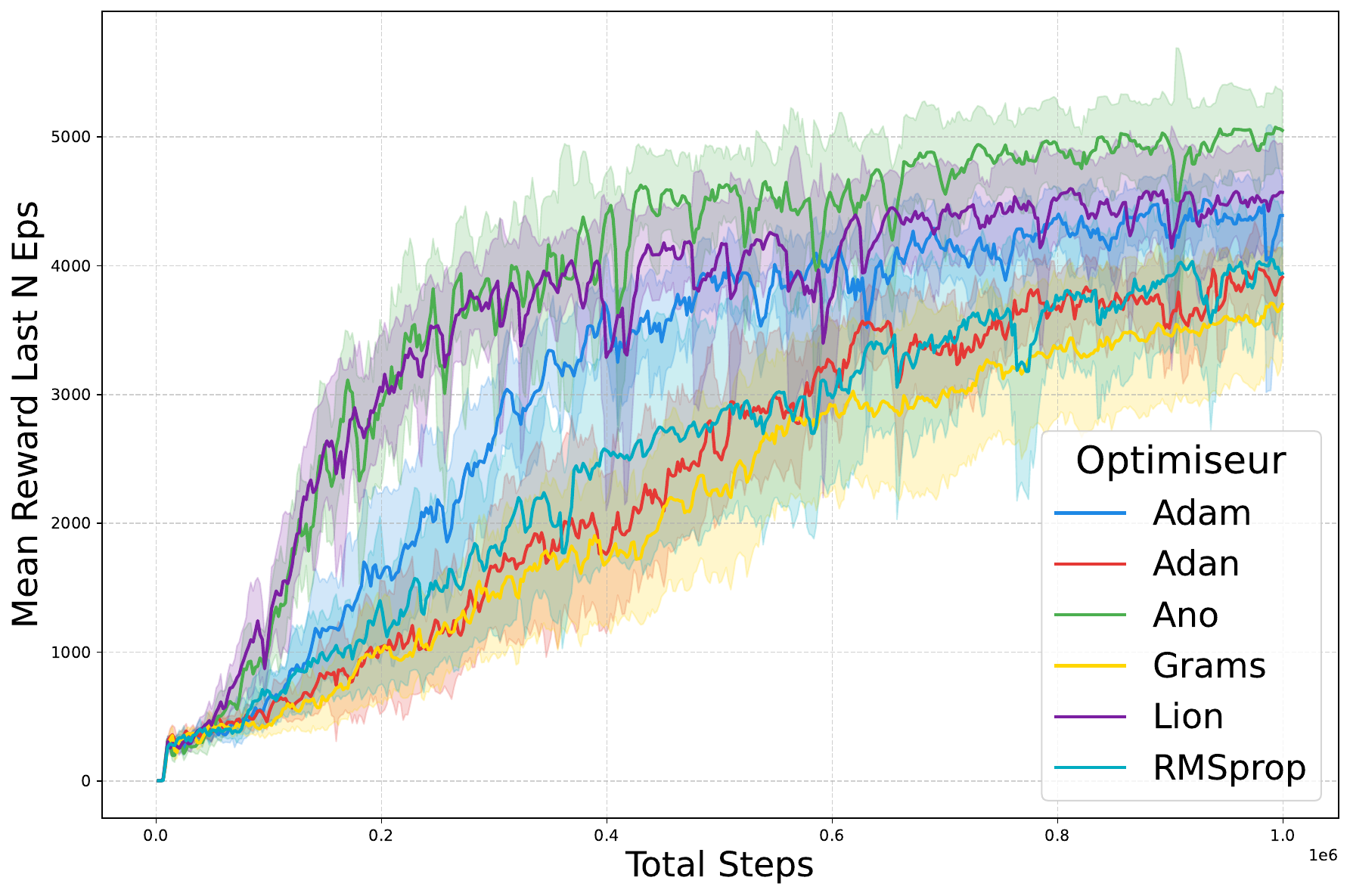}
        \caption{Walker2d-v5}
        \label{fig:walker2d-plot}
    \end{subfigure}
    \begin{subfigure}[b]{0.30\textwidth}
        \centering
        \includegraphics[width=\textwidth]{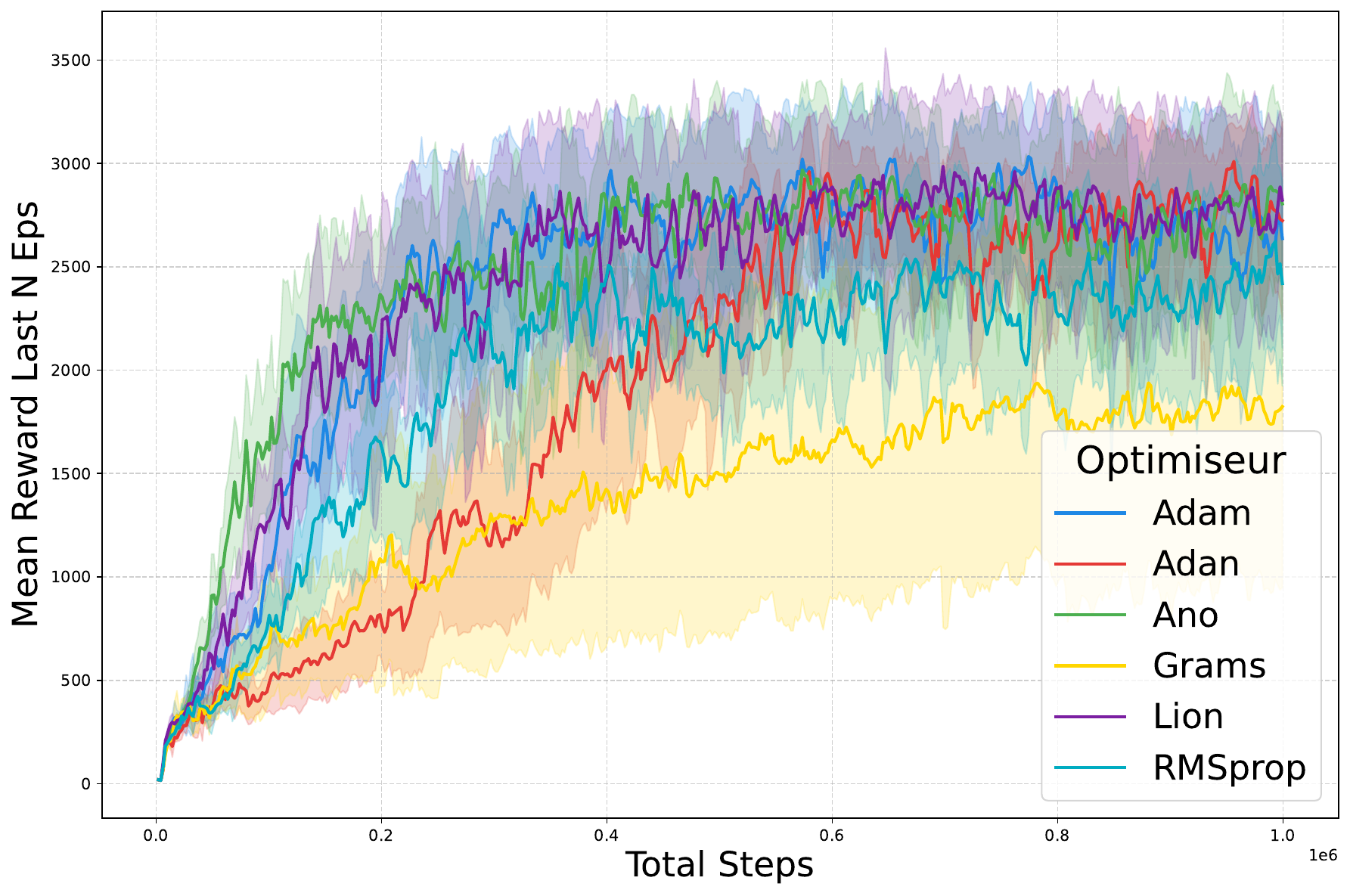}
        \caption{Hopper-v5}
        \label{fig:hopper-plot}
    \end{subfigure}
    \caption{Rewards over time for several MuJoCo environments, with baselines and 95\% confidence intervals. The green curve corresponds to \textbf{Ano (ours)}.}
    \label{fig:MuJoCo_rewardsovertime}
\end{figure}

\begin{table}[h]
\centering
\resizebox{\textwidth}{!}{
\begin{tabular}{l|ccccc|cc}
\toprule
Optimizers & HalfCheetah & Ant & Humanoid & Walker2d & Hopper & Mean Rank & Norm. Avg \\
\midrule
\textit{Default} \\
Adam  & $10549.48 \pm 721.55$ & $4336.64 \pm 698.72$ & $\mathbf{5357.14} \pm \mathbf{211.97}$ & $4462.51 \pm 588.77$ & $3164.71 \pm 600.48$ & $3.4$ & $90.66$ \\
RMSprop     & $10506.23 \pm 852.19$  & $4234.37 \pm 763.65$  & $5395.51 \pm 126.80$  & $4160.06 \pm 480.62$  & $2973.86 \pm 571.05$  & $5.6$ & $87.83$ \\
Adan  & $7805.20 \pm 1154.02$ & $2985.19 \pm 1018.79$ & $5080.74 \pm 305.26$ & $4092.13 \pm 379.92$ & $3222.62 \pm 235.25$ & $5.6$ & $78.38$ \\
Lion  & $9527.96 \pm 805.42$ & $4948.26 \pm 243.05$ & $98.22 \pm 32.33$ & $4612.63 \pm 367.77$ & $3087.27 \pm 628.06$ & $4.4$ & $71.74$ \\
Grams & $6782.60 \pm 715.12$ & $3207.30 \pm 531.06$ & $5104.10 \pm 692.14$ & $3656.66 \pm 658.82$ & $1475.34 \pm 927.22$ & $6.4$ & $65.88$ \\
\textbf{Ano (Ours)}    & $\mathbf{10864.09} \pm \mathbf{1052.24}$ & $\mathbf{5285.44} \pm \mathbf{729.86}$ & $5255.62 \pm 815.92$ & $\mathbf{5227.86} \pm \mathbf{436.49}$ & $\mathbf{3535.32} \pm \mathbf{780.96}$ & $\mathbf{1.4}$ & $\mathbf{99.48}$ \\
\textbf{Anolog (Ours)} & $10557.05 \pm 560.70$ & $5089.12 \pm 522.94$ & $5242.78 \pm 173.98$ & $4606.02 \pm 478.36$ & $3314.12 \pm 539.95$ & $2.6$ & $94.50$ \\
\addlinespace
\textit{Best Version} \\
Adam [Default] & $10549.48 \pm 721.55$ & $4336.64 \pm 698.72$ & $\mathbf{5357.14} \pm \mathbf{211.97}$ & $4462.51 \pm 588.77$ & $3164.71 \pm 600.48$ & $4.6$ & $90.38$ \\
RMSprop [Default] & $10506.23 \pm 852.19$  & $4234.37 \pm 763.65$  & $5395.51 \pm 126.80$  & $4160.06 \pm 480.62$  & $2973.86 \pm 571.05$  & $5.6$ & $87.83$ \\
Adan [Tuned] & $10822.40 \pm 475.75$ & $5239.69 \pm 270.96$ & $4792.62 \pm 904.44$ & $4686.83 \pm 502.28$ & $3514.42 \pm 143.57$ & $3.2$ & $95.01$ \\
Lion [Tuned] & $10482.06 \pm 1018.86$ & $4848.41 \pm 821.79$ & $1349.15 \pm 1322.56$ & $4876.76 \pm 253.22$ & $\mathbf{3592.87} \pm \mathbf{70.26}$ & $4.2$ & $81.30$ \\
Grams [Tuned] & $10533.70 \pm 866.69$ & $4607.59 \pm 505.08$ & $5147.04 \pm 487.55$ & $4644.45 \pm 498.08$ & $3147.82 \pm 605.03$ & $5.0$ & $91.20$ \\
\textbf{Ano (Ours)} [Default] & $\mathbf{10864.09} \pm \mathbf{1052.24}$ & $\mathbf{5285.44} \pm \mathbf{729.86}$ & $5255.62 \pm 815.92$ & $\mathbf{5227.86} \pm \mathbf{436.49}$ & $3535.32 \pm 780.96$ & $\mathbf{1.6}$ & $\mathbf{99.16}$ \\
\textbf{Anolog (Ours)} [Default] & $10557.05 \pm 560.70$  & $5089.12 \pm 522.94$  & $5242.78 \pm 173.98$  & $4606.02 \pm 478.36$  & $3314.12 \pm 539.95$  & $2.6$ & $94.20$ \\
\bottomrule
\end{tabular}
}
\caption{Comparison of the IQM $\pm$ CI95\% of different optimizers across environments.}
\label{tab:sac-comparative}
\end{table}

As summarized in Table~\ref{tab:sac-comparative}, Ano performs favorably compared to Adam and other baselines across the MuJoCo tasks. On average, it achieves a \textbf{+10\%} improvement in normalized score\footnote{The normalized average is obtained by linearly rescaling each score between the minimum and maximum values observed across optimizers, followed by averaging. Complete normalized results are reported in Table~\ref{tab:drl-norm}.}, both under default and tuned hyperparameters. Without tuning, Ano ranks first in 4    out of 5 tasks; with best version, it remains the top optimizer in 3 out of 5 tasks. Although not always the best performer, Ano consistently ranks among the strongest optimizers, with its scores typically within or close to the 95\% confidence intervals of the best baselines. Figure~\ref{fig:MuJoCo_rewardsovertime} shows that Ano reaches the final performance of Adam using approximately 50–70\% fewer training steps, except for \textit{Humanoid}. To address potential concerns about hyperparameter tuning, we evaluated the sensitivity to learning rate and momentum coefficients on a 100k-step \textit{HalfCheetah} proxy (Figure~\ref{fig:resume_anoadam_hyperparamscomparison}; see \ref{app:tuning_experiments} for full details). Ano shows lower sensitivity than Adam to both learning rate and betas, suggesting that its performance gains are not solely due to more favorable hyperparameter choices.

\begin{figure}[H]
    \centering
    \includegraphics[width=0.9\textwidth]{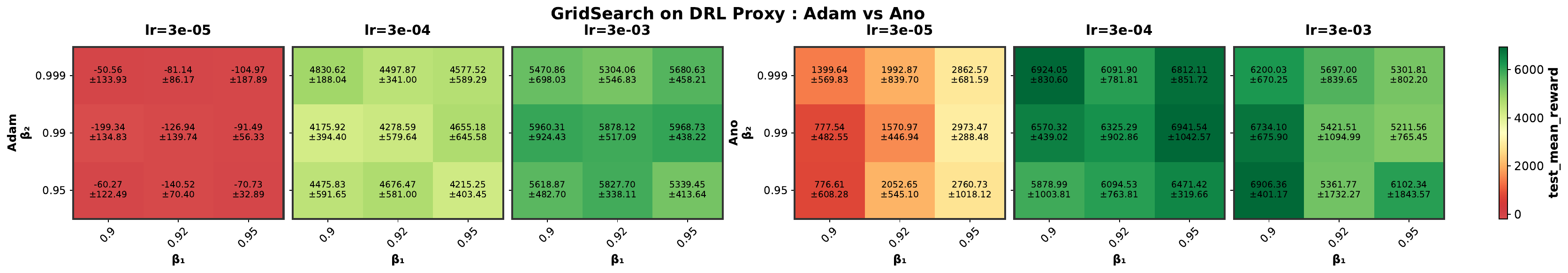}
    \caption{Hyperparameter robustness on a MuJoCo proxy (HalfCheetah with SAC). \textbf{Adam} on the left, \textbf{Ano(ours)} on the right.}
    \label{fig:resume_anoadam_hyperparamscomparison}
\end{figure}

\paragraph{Proximal Policy Optimization.}
To assess the generality of our findings, we additionally evaluate a discrete-action variant of Proximal Policy Optimization (PPO)~\citep{schulman2017proximalpolicyoptimizationalgorithms} on the Atari Learning Environment (ALE)~\citep{bellemare2013arcade}. For compute efficiency, we report results on the \textit{Atari-5} subset proposed by Aitchison et al.~\citep{aitchison2023atari}, which explains $98.4\%$ of the variance in full-suite performance ($R^2{=}0.984$). We use the reference PPO implementation from CleanRL~\citep{huang2022cleanrl} and keep its default network and optimization hyperparameters (full list in Appendix~\ref{app:all_hyperparameters_settings} - tab \ref{tab:ppo_hparams_main}). Environments are instantiated via EnvPool~\citep{weng2022envpoolhighlyparallelreinforcement} for fast batched simulation. Observations are resized to $84{\times}84$, converted to grayscale, and stacks of 4 consecutive frames are fed to the agent. We apply action repeat of 4, up to 30 no-op actions at reset, the ALE sticky-action protocol with repeat probability $0.25$~\citep{machado2018revisiting}, and FireReset where required. We used the full action set. During training, rewards are clipped to $[-1,1]$; evaluation uses unclipped rewards. We train for $10$M agent steps (with action repeat 4, i.e., $\approx 40$M ALE frames), evaluating every $200$k steps on 50 runs. Each checkpoint is evaluated with the same wrappers as training (except reward clipping). We report the average final score on the last evaluation. The normalization average is computed in the same ways that previous part\footnote{We normalized each score by environment by the theoretically minimum (-18 for DoubleDunk, 0 for the others) and the maximum score obtained by the baselines, we then compute the mean average normalized score.}.

\begin{figure}[h]
    \centering
    \begin{subfigure}[b]{0.30\textwidth}
        \centering
        \includegraphics[width=\textwidth]{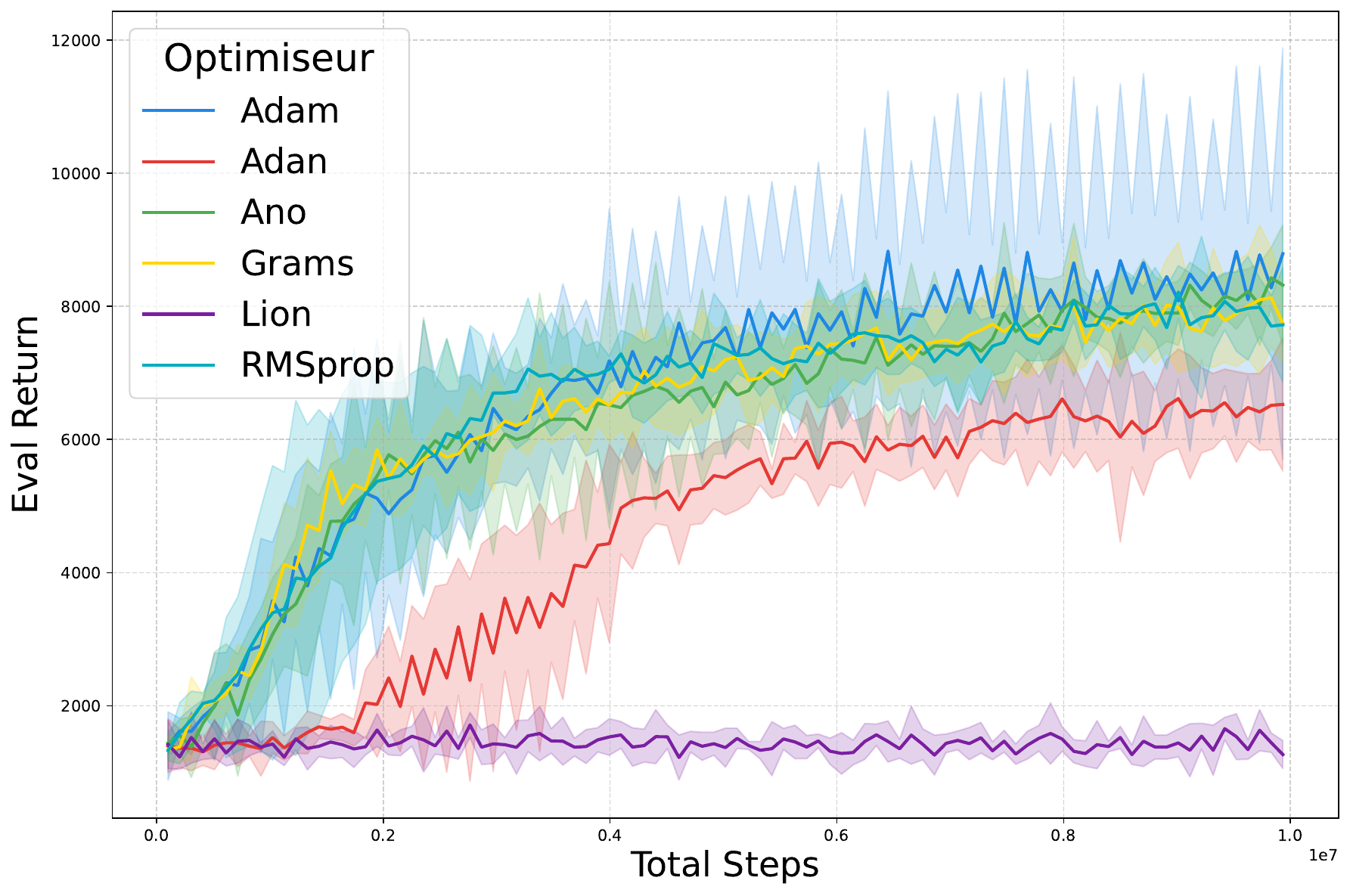}
        \caption{BattleZone-v5}
        \label{fig:battlezone-plot}
    \end{subfigure}
    \hfill
    \begin{subfigure}[b]{0.30\textwidth}
        \centering
        \includegraphics[width=\textwidth]{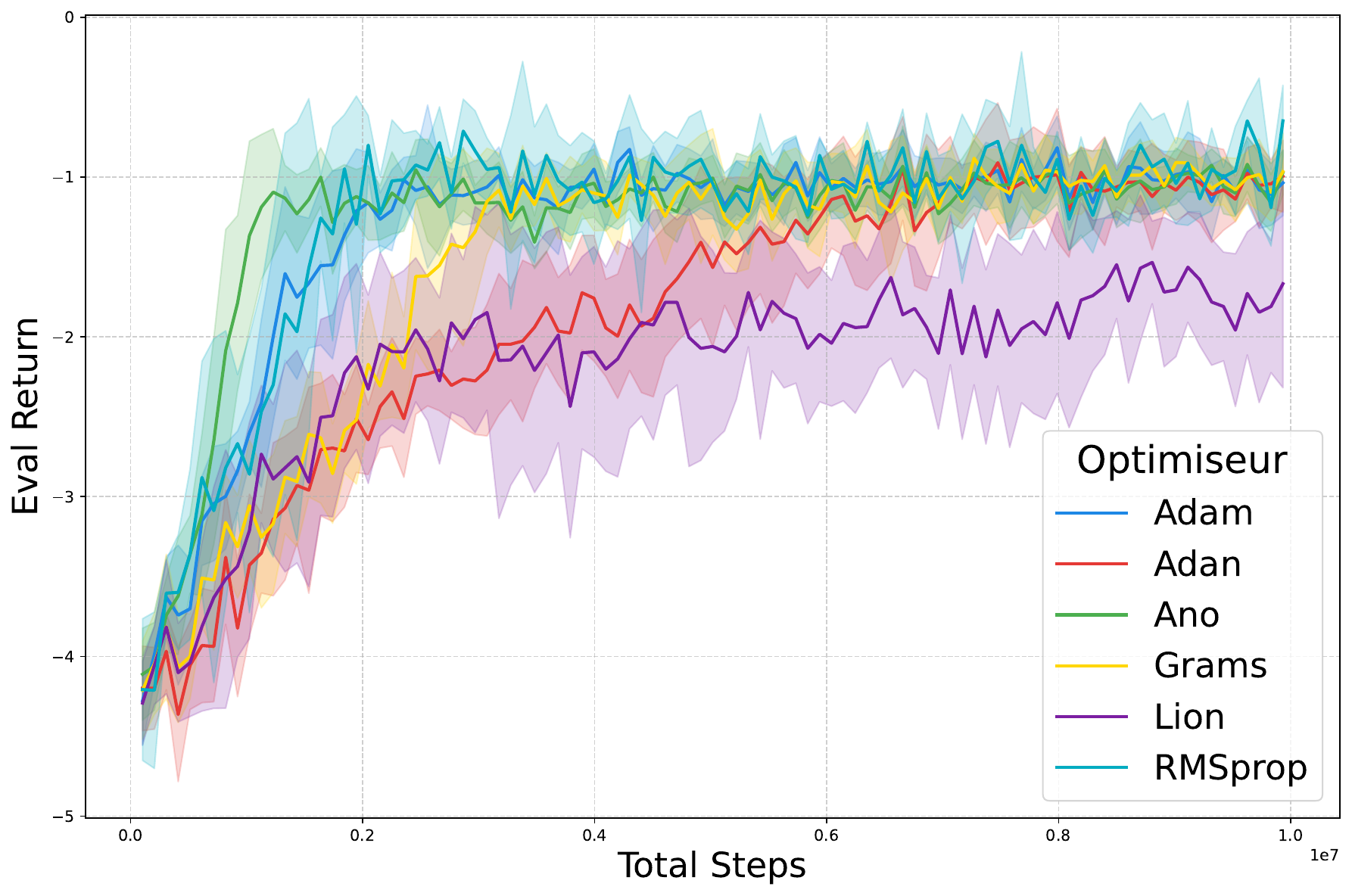}
        \caption{DoubleDunk-v5}
        \label{fig:doubledunk-plot}
    \end{subfigure}
    \hfill
    \begin{subfigure}[b]{0.30\textwidth}
        \centering
        \includegraphics[width=\textwidth]{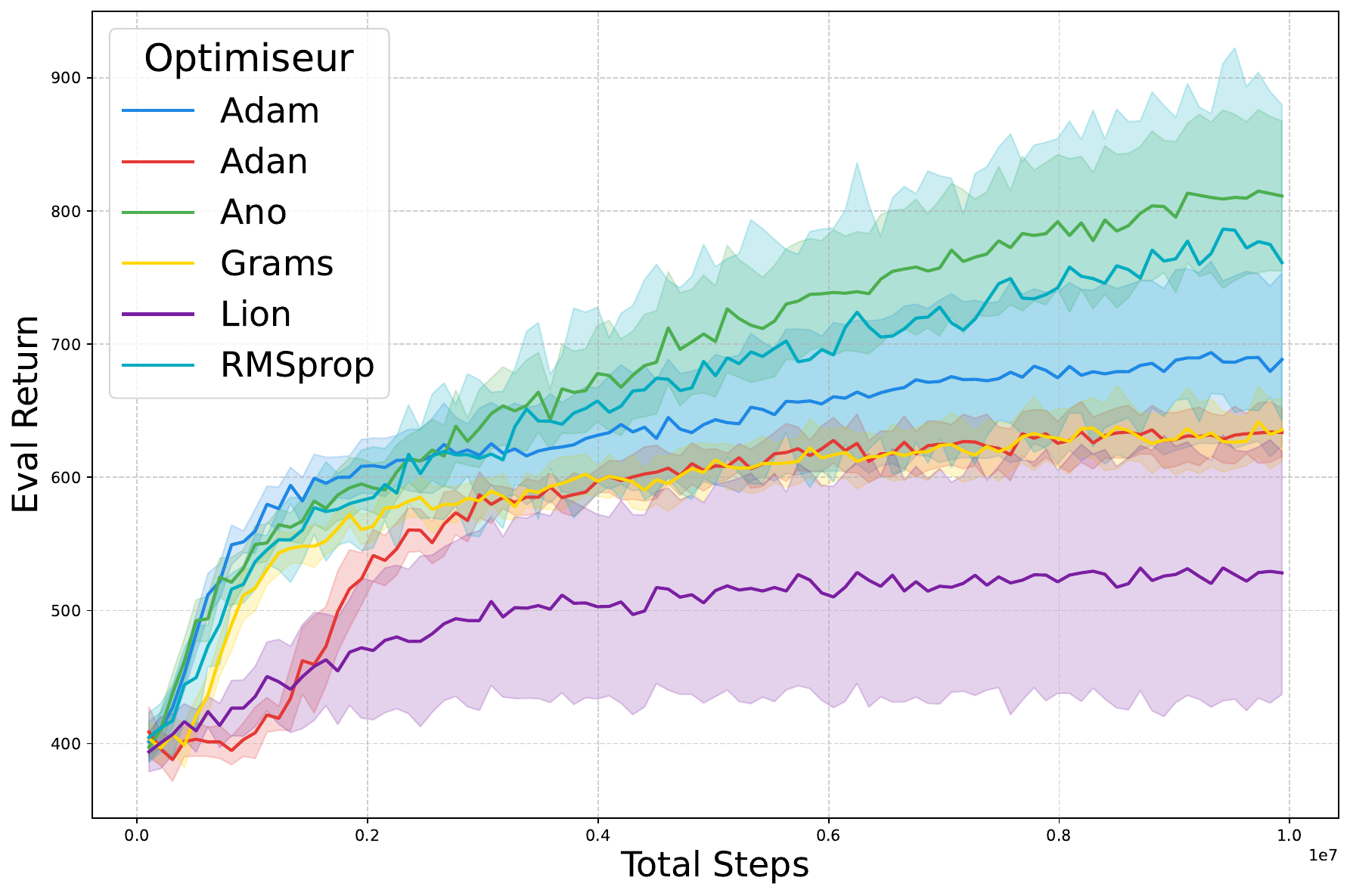}
        \caption{NameThisGame-v5}
        \label{fig:namethisgame-plot}
    \end{subfigure}
    \hfill
    \begin{subfigure}[b]{0.30\textwidth}
        \centering
        \includegraphics[width=\textwidth]{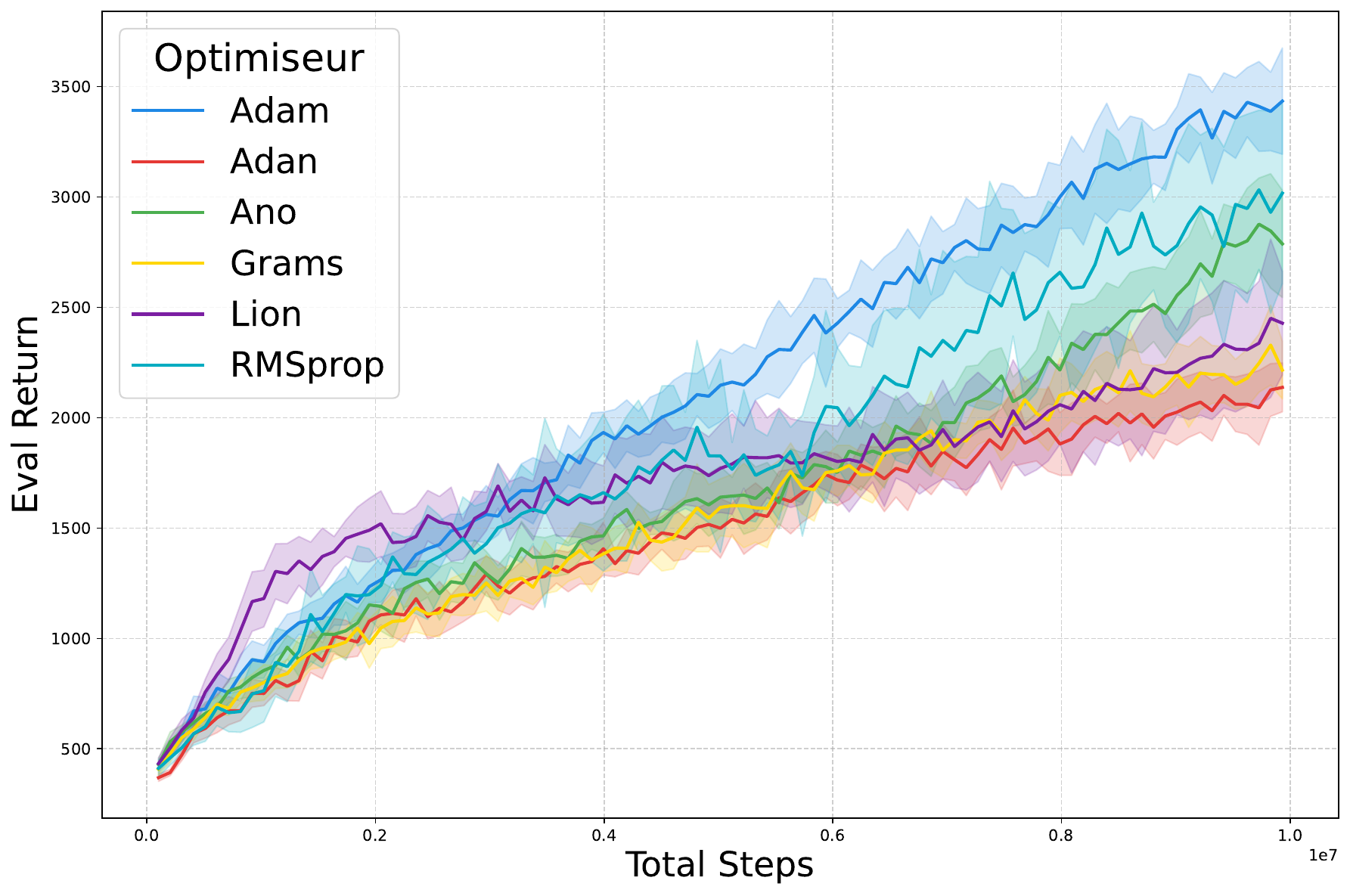}
        \caption{Phoenix-v5}
        \label{fig:phoenix-plot}
    \end{subfigure}
    \begin{subfigure}[b]{0.30\textwidth}
        \centering
        \includegraphics[width=\textwidth]{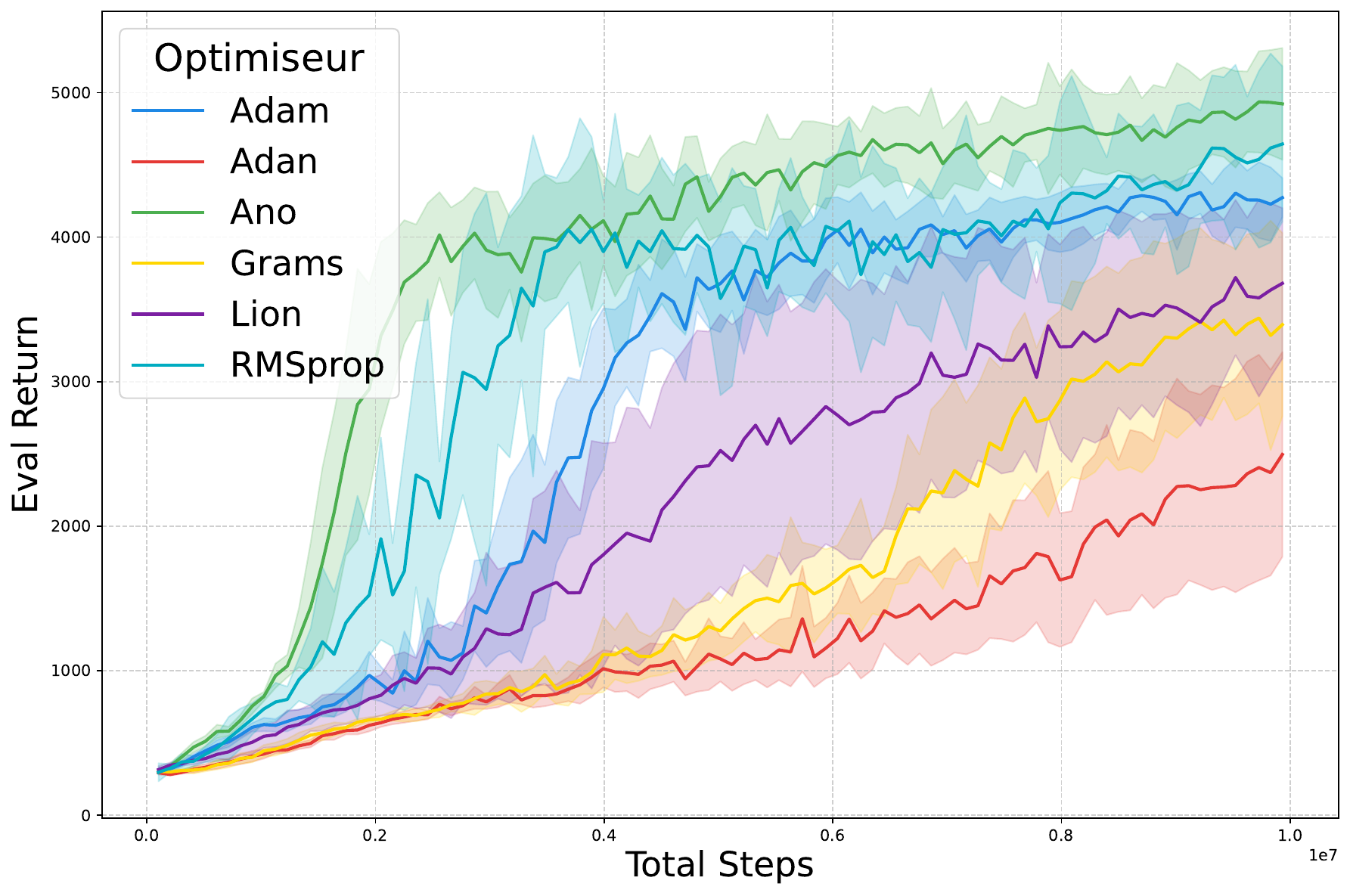}
        \caption{Qbert-v5}
        \label{fig:qbert-plot}
    \end{subfigure}
    \caption{Rewards over time for Atari5 Benchmark, with baselines and 95\% confidence intervals. The green curve corresponds to \textbf{Ano (ours)}.}
    \label{fig:atari-rewardsovertime}
\end{figure}

\begin{table}[h]
\centering
\resizebox{\textwidth}{!}{
\begin{tabular}{l|ccccc|cc}
\toprule
Optimizers & BattleZone-v5 & DoubleDunk-v5 & NameThisGame-v5 & Phoenix-v5 & Qbert-v5 & Mean Rank & Norm. Avg \\
\midrule
\textit{Default} \\
Adam        & $7615.00 \pm 1299.89$ & $-1.08 \pm 0.21$ & $665.35 \pm 64.89$ & $\textbf{3443.30} \pm \textbf{240.68}$ & $4257.80 \pm 135.90$ & $4.4$ & $87.54$ \\
RMSprop     & $7686.67 \pm 859.43$  & $\textbf{-0.67} \pm \textbf{0.22}$ & $798.00 \pm 118.66$ & $3031.13 \pm 410.22$ & $4585.67 \pm 538.44$ & $2.4$ & $90.09$ \\
Adan        & $6480.00 \pm 702.14$  & $-0.91 \pm 0.23$ & $638.35 \pm 18.08$  & $2106.90 \pm 110.74$ & $2665.00 \pm 708.69$ & $5.4$ & $74.11$ \\
Lion        & $1392.00 \pm 139.66$  & $-1.46 \pm 0.65$ & $508.15 \pm 91.28$  & $2432.35 \pm 234.61$ & $3768.00 \pm 520.95$ & $6.8$ & $61.36$ \\
Grams       & $7908.00 \pm 687.61$  & $-0.93 \pm 0.24$ & $633.80 \pm 23.52$  & $2234.40 \pm 130.73$ & $3670.12 \pm 634.75$ & $4.6$ & $82.41$ \\
\textbf{Ano (Ours)}    & $\mathbf{8095.00} \pm \mathbf{494.70}$ & $-0.97 \pm 0.14$ & $\mathbf{845.35} \pm \mathbf{56.39}$ & $2813.00 \pm 243.04$ & $\mathbf{4828.80} \pm \mathbf{386.73}$ & $\mathbf{2.2}$ & $\mathbf{95.99}$ \\
\textbf{Anolog (Ours)} & $7485.00 \pm 1010.66$ & $-0.98 \pm 0.14$ & $751.05 \pm 74.60$ & $2983.00 \pm 236.42$ & $4773.75 \pm 602.39$ & $3.6$ & $93.00$ \\
\addlinespace
\textit{Best Version} \\
Adam [Baseline] & $7615.00 \pm 1299.89$ & $-1.08 \pm 0.21$ & $665.35 \pm 64.89$  & $\textbf{3443.30} \pm \textbf{240.68}$ & $4257.80 \pm 135.90$ & $4.4$ & $87.54$ \\
RMSprop [Baseline] & $7686.67 \pm 859.43$  & $\textbf{-0.67} \pm \textbf{0.22}$ & $798.00 \pm 118.66$  & $3031.13 \pm 410.22$ & $4585.67 \pm 538.44$ & $2.4$ & $90.09$ \\
Adan [Tuned] & $4840.00 \pm 2601.35$ & $-0.95 \pm 0.23$ & $754.20 \pm 51.44$   & $2647.20 \pm 534.22$ & $4524.75 \pm 448.74$ & $4.4$ & $79.67$ \\
Lion [Baseline] & $1392.00 \pm 139.66$  & $-1.46 \pm 0.65$ & $508.15 \pm 91.28$   & $2432.35 \pm 234.61$ & $3768.00 \pm 520.95$ & $6.8$ & $61.36$ \\
Grams [Tuned] & $7715.00 \pm 627.92$  & $-1.35 \pm 0.75$ & $690.40 \pm 76.27$   & $1989.15 \pm 201.64$ & $5049.25 \pm 624.43$ & $4.4$ & $82.26$ \\
\textbf{Ano (Ours)} [Tuned] & $\mathbf{8625.00} \pm \mathbf{1870.44}$ & $-0.91 \pm 0.21$ & $\mathbf{828.10} \pm \mathbf{67.66}$ & $2824.85 \pm 226.30$ & $\mathbf{5960.88} \pm \mathbf{912.36}$ & $\mathbf{1.8}$ & $\mathbf{96.13}$ \\
\textbf{Anolog (Ours)} [Baseline] & $7485.00 \pm 1010.66$ & $-0.98 \pm 0.14$ & $751.05 \pm 74.60$ & $2983.00 \pm 236.42$ & $4773.75 \pm 602.39$ & $3.8$ & $88.48$ \\
\bottomrule
\end{tabular}
}
\caption{Comparison of the IQM $\pm$ CI95\% of different optimizers across Atari environments.}
\label{tab:drl-comparative-atari}
\end{table}

As shown in Table~\ref{tab:drl-comparative-atari}, Ano and RMSprop perform strongest overall among the baselines. In the default setting, their mean ranks are 2.2 and 2.4; with the best version, 1.8 and 2.4, respectively. Ano achieves the highest average normalized score and mean rank in both regimes, with approximately 6–7\% higher normalized average than RMSprop and 10\% higher than Adam. Notably, Ano outperforms Adam on \textit{BattleZone}, \textit{Name This Game}, and \textit{Q*bert}, whereas Adam and RMSprop perform best on \textit{Phoenix}. For \textit{DoubleDunk}, all optimizers (except Lion) plateau at similar levels (Fig.~\ref{fig:doubledunk-plot}), so no clear winner emerges.

\section{Ablation Study}\label{sec:ablation}

We conduct ablation studies on Ano and on its variant Anolog to quantify the contribution of each design component and, to justify adopting a logarithmic momentum schedule instead of a theoretically motivated square-root schedule. 

We conduct ablation studies on Ano and its variant Anolog to quantify the contribution of each design component and to justify using a logarithmic momentum schedule rather than the theoretically motivated square-root schedule.
Table~\ref{tab:ablation-study-ano} summarizes all ablated variants. To provide a comprehensive evaluation, we compare performance on four benchmarks: \textit{HalfCheetah} from MuJoCo \citep{todorov2012mujoco}, CIFAR100 \citep{krizhevsky2009learning}, and two tasks from the GLUE benchmark: the small and noisy MRPC task, and the larger, more stable SST2 task. We follow the same experimental protocols as in Section~\ref{sec:experiments}, except that for \textit{HalfCheetah} we train for 500k steps.

\begin{table}[H]
\centering
\resizebox{\textwidth}{!}{%
\begin{tabular}{l|cccccc|cccc}
\toprule
\textbf{Optimizer} 
& \shortstack{Second Mom.\\Rule} 
& \shortstack{Grad.\\Norm} & \shortstack{Mom.\\Norm} & \shortstack{Mom.\\Dir.} 
& \shortstack{Decoup.\\WD} 
& \shortstack{$\beta_{1,k}$} 
& \shortstack{Score\\DRL} 
& \shortstack{Acc.~(\%)\\CIFAR-100} 
& \shortstack{Acc.~(\%)\\MRPC} 
& \shortstack{Acc.~(\%)\\SST-2} \\ 
\midrule
\textit{Ano ablation} \\
Adam                                  & Adam & \xmark & \checkmark & \checkmark & \checkmark & $\beta_{1}$ & $7480.55 \pm 1323.36$ & $69.84 \pm 0.22$ & $85.93 \pm 0.92$ & $\textbf{93.03} \pm \textbf{0.30}$ \\
YogiTweaked                           & Yogi+$\beta_2$-decay & \xmark & \checkmark & \checkmark & \checkmark & $\beta_{1}$ & $8540.52 \pm 671.22$ & $68.62 \pm 2.36$ & $85.25 \pm 1.22$ & $92.75 \pm 0.32$ \\
Grams                                 & Adam & \xmark & \checkmark & \xmark & \checkmark & $\beta_{1}$ & $5567.12 \pm 782.37$  & $70.20 \pm 0.17$ & $82.25 \pm 0.74$ & $92.29 \pm 0.23$ \\
YogiSignum                            & Yogi+$\beta_2$-decay & \xmark & \xmark     & \checkmark & \checkmark & $\beta_{1}$ & $-285.58 \pm 41.11$   & $3.99 \pm 2.01$  & $68.38 \pm 0.00$ & $50.92 \pm 0.00$ \\
Signum                                & \xmark & \xmark & \xmark     & \checkmark & \checkmark & $\beta_{1}$ & $9393.64 \pm 1399.78$ & $65.11 \pm 0.90$ & $86.42 \pm 0.72$ & $90.41 \pm 0.30$ \\
SignumGrad                            & \xmark & \checkmark & \xmark & \checkmark & \checkmark & $\beta_{1}$ & -- & $53.93 \pm 0.68$ & $68.38 \pm 0.00$ & $53.33 \pm 2.62$ \\
AdamGrad                              & Adam & \checkmark & \xmark & \checkmark & \checkmark & $\beta_{1}$ & $9855.19 \pm 1173.19$ & $70.30 \pm 0.38$ & $86.96 \pm 0.85$ & $92.71 \pm 0.45$ \\
AnoWoTweak                            & Yogi & \checkmark & \xmark & \checkmark & \checkmark & $\beta_{1}$ & $9053.10 \pm 792.13$  & $\textbf{70.32} \pm \textbf{1.20}$ & $\textbf{87.06} \pm \textbf{0.69}$ & $92.80 \pm 0.45$ \\
\textbf{Ano}                          & Yogi+$\beta_2$-decay & \checkmark & \xmark & \checkmark & \checkmark & $\beta_{1}$ & $\mathbf{10520.00 \pm 416.07}$ & $69.74 \pm 0.45$ & $86.76 \pm 0.63$ & $92.52 \pm 0.31$ \\
\addlinespace
\textit{Anolog ablation} \\
Anoall                                & Yogi+$\beta_2$-decay & \checkmark & \xmark & \checkmark & \checkmark & $1-1/k$ & $-221.45 \pm 22.25$  & $29.48 \pm 2.40$ & $68.38 \pm 0.00$ & $52.22 \pm 1.88$ \\
Anosqrt                               & Yogi+$\beta_2$-decay & \checkmark & \xmark & \checkmark & \checkmark & $1-1/\sqrt{k}$ & $8750 \pm 860.50$  & $\textbf{67.26} \pm \textbf{0.41}$ & $\textbf{86.18} \pm \textbf{1.08}$ & $91.74 \pm 0.53$ \\
\textbf{Anolog}                       & Yogi+$\beta_2$-decay & \checkmark & \xmark & \checkmark & \checkmark & $1-1/\log k$ & $\mathbf{9472.73 \pm 968.26}$ & $67.00 \pm 0.80$ & $85.25 \pm 1.79$ & $\textbf{92.78} \pm \textbf{0.16}$ \\
\bottomrule
\end{tabular}}
\caption{Ablation of our proposed optimizer(\textbf{Ano}) and its extension (\textbf{Anolog}).  Columns on the left indicate which components are active; columns on the right report mean performance $\pm$ 95\% CI}
\label{tab:ablation-study-ano}
\end{table}

As shown in Table~\ref{tab:ablation-study-ano}, \textbf{Ano} achieves the highest mean return in deep reinforcement learning, improving by roughly 7\% over the same algorithm with Adam-style second moments and about 15\% over Ano with Yogi-style second moments, while staying within 1\% of the best accuracy on all supervised learning tasks. Using only the sign of the momentum (e.g., Signum, AdamGrad, AnoWoDecay) also improves DRL performance, supporting our design choice to decouple sign and magnitude: this enables larger update steps, which are particularly beneficial in noisy or non-stationary environments. Performance drops when either gradient normalization (SignumGrad) or gradient magnitude (YogiSignum) is removed, underscoring their complementary roles. For momentum schedules, the logarithmic schedule improves DRL return over the $\sqrt{k}$ schedule while staying within the 95\% confidence interval on other tasks, motivating its inclusion in the final design.

\section{Limitations and Discussion}
\label{sec:limits-discussion}
Through our design and empirical analysis of Ano, we identified three main limitations:
First, the choice of $\beta_2$-decay appears particularly beneficial in reinforcement learning or highly non-stationary loss landscapes, where older gradients can be misleading and a rapid adaptation is crucial. However, in more stationary settings, such as classical supervised learning in CV and NLP, we observed that the variance estimate in vanilla Yogi often leads to more stable and effective training. Our focus on noisy, non-stationary environments motivates this design choice, though we acknowledge that its relevance to more conventional settings remains an open question.
Second, by construction, Ano favors larger step sizes to improve reactivity. While this design is advantageous in non-stationary contexts, it can also introduce instability. For example, our experiments with Nesterov-style acceleration, inspired by Adan, amplified rather than mitigated this issue.
Third, our experiments on classical CV and NLP tasks remain limited in scale, as Ano was primarily designed for highly non-stationary and noisy environments. In more stationary settings with longer training horizons, we observed that Adam can sometimes achieve better stability due to its smaller update steps. While these results suggest that Ano’s benefits are not restricted to DRL, assessing its relevance to large-scale CV or NLP tasks lies beyond the current scope and is left for future work.

\section{Conclusion}
We introduced \textbf{Ano}, an alternative to momentum-based adaptive optimizers that decouples direction and magnitude to improve robustness in noisy and non-stationary settings. Under standard smoothness and bounded-noise assumptions, we derive non-asymptotic guarantees comparable to existing analyses of sign-based methods (e.g., Signum, Lion) under similar decay schedules. Empirically, Ano achieves notable improvements in reinforcement learning and noisy NLP tasks while remaining competitive on low-noise benchmarks. Future work will focus on developing variance estimators tailored to supervised learning, integrating Nesterov-style look-ahead, and enhancing stability in long, stationary training regimes.\footnote{LLMs (GPT, Gemini) were used for minor text editing, LaTeX formatting, informal feedback on early drafts, and retrieval of related work and references; all research ideas, analyses, and conclusions remain solely those of the authors.}

\section{Reproducibility Statement}

All datasets used in this work are publicly available. The full source code, including training, preprocessing, and result visualization scripts, as well as all experiment logs, is released in an anonymous repository\footnote{https://anonymous.4open.science/r/ano-optimizer-1645/README.md}. The optimizer is also available as a pip package (PyTorch, TensorFlow, JAX) to facilitate implementation, but it isn't include in the source code for double bind review. Data preprocessing details, hyperparameter grids, and training protocols are described in Section~\ref{sec:experiments} and provided in the source code. All experiments were run with fixed random seeds on a workstation with an RTX 5090 GPU and an Intel Core Ultra 9 CPU using CUDA~12.9 and PyTorch~2.9.0.


\bibliography{iclr2025_conference}

@inproceedings{kingma2015adam,
  title     = {Adam: A Method for Stochastic Optimization},
  author    = {Kingma, Diederik P. and Ba, Jimmy},
  booktitle = {International Conference on Learning Representations},
  year      = {2015},
}

@article{xie2024adan,
  title={Adan: Adaptive nesterov momentum algorithm for faster optimizing deep models},
  author={Xie, Xingyu and Zhou, Pan and Li, Huan and Lin, Zhouchen and Yan, Shuicheng},
  journal={IEEE Transactions on Pattern Analysis and Machine Intelligence},
  volume={46},
  number={12},
  pages={9508--9520},
  year={2024}
}

@article{chen2023symbolic,
  title={Symbolic discovery of optimization algorithms},
  author={Chen, Xiangning and Liang, Chen and Huang, Da and Real, Esteban and Wang, Kaiyuan and Pham, Hieu and Dong, Xuanyi and Luong, Thang and Hsieh, Cho-Jui and Lu, Yifeng and others},
  journal={Advances in neural information processing systems},
  volume={36},
  pages={49205--49233},
  year={2023}
}

@misc{tieleman2012rmsprop,
  title={Lecture 6.5-RMSprop: Divide the gradient by a running average of its recent magnitude},
  author={Tieleman, Tijmen and Hinton, Geoffrey},
  year         = {2012},
  howpublished = {Coursera: Neural Networks for Machine Learning},
  note         = {Lecture 6.5, slide 29}
}

@inproceedings{mnih2016asynchronous,
  title={Asynchronous methods for deep reinforcement learning},
  author={Mnih, Volodymyr and Badia, Adria Puigdomenech and Mirza, Mehdi and Graves, Alex and Lillicrap, Timothy and Harley, Tim and Silver, David and Kavukcuoglu, Koray},
  booktitle={International conference on machine learning},
  pages={1928--1937},
  year={2016}
}

@inproceedings{balles2018dissecting,
  title={Dissecting adam: The sign, magnitude and variance of stochastic gradients},
  author={Balles, Lukas and Hennig, Philipp},
  booktitle={International Conference on Machine Learning},
  pages={404--413},
  year={2018}
}

@inproceedings{he2016deep,
  title={Deep residual learning for image recognition},
  author={He, Kaiming and Zhang, Xiangyu and Ren, Shaoqing and Sun, Jian},
  booktitle={Proceedings of the IEEE conference on computer vision and pattern recognition},
  pages={770--778},
  year={2016}
}

@inproceedings{Zaheer2018Adaptive,
  title     = {Adaptive Methods for Nonconvex Optimization},
  author    = {Zaheer, Manzil and Reddi, Sashank~J. and Sachan, Devendra and Kale, Satyen and Kumar, Sanjiv},
  booktitle = {Advances in Neural Information Processing Systems 31},
  pages     = {9793--9803},
  year      = {2018},
}

@inproceedings{bernstein2018signsgd,
  title={signSGD: Compressed optimisation for non-convex problems},
  author={Bernstein, Jeremy and Wang, Yu-Xiang and Azizzadenesheli, Kamyar and Anandkumar, Animashree},
  booktitle={International conference on machine learning},
  pages={560--569},
  year={2018}
}

@inproceedings{zagoruyko2017wideresidualnetworks,
    author={Sergey Zagoruyko and Nikos Komodakis},
    title={Wide Residual Networks}, 
    booktitle = {British Machine Vision Conference},
    year = {2016}
}

@article{franccois2018introduction,
  title={An introduction to deep reinforcement learning},
  author={Fran{\c{c}}ois-Lavet, Vincent and Henderson, Peter and Islam, Riashat and Bellemare, Marc G and Pineau, Joelle and others},
  journal={Foundations and Trends{\textregistered} in Machine Learning},
  volume={11},
  number={3-4},
  pages={219--354},
  year={2018}
}

@inproceedings{todorov2012mujoco,
  title={Mujoco: A physics engine for model-based control},
  author={Todorov, Emanuel and Erez, Tom and Tassa, Yuval},
  booktitle={2012 IEEE/RSJ international conference on intelligent robots and systems},
  pages={5026--5033},
  year={2012}
}

@article{towers2024gymnasiumstandardinterfacereinforcement,
  title={Gymnasium: A standard interface for reinforcement learning environments},
  author={Towers, Mark and Kwiatkowski, Ariel and Terry, Jordan and Balis, John U and De Cola, Gianluca and Deleu, Tristan and Goul{\~a}o, Manuel and Kallinteris, Andreas and Krimmel, Markus and KG, Arjun and others},
  journal={arXiv preprint arXiv:2407.17032},
  year={2024}
}

@inproceedings{devlin2019bert,
  title={Bert: Pre-training of deep bidirectional transformers for language understanding},
  author={Devlin, Jacob and Chang, Ming-Wei and Lee, Kenton and Toutanova, Kristina},
  booktitle={Proceedings of the 2019 conference of the North American chapter of the association for computational linguistics: human language technologies, volume 1 (long and short papers)},
  pages={4171--4186},
  year={2019}
}

@inproceedings{
wang2019gluemultitaskbenchmarkanalysis,
title={{GLUE}: A Multi-Task Benchmark and Analysis Platform for Natural Language Understanding},
author={Alex Wang and Amanpreet Singh and Julian Michael and Felix Hill and Omer Levy and Samuel R. Bowman},
booktitle={International Conference on Learning Representations},
year={2019}
}

@inproceedings{mosbach2021on,
title={On the Stability of Fine-tuning BERT: Misconceptions, Explanations, and Strong Baselines},
author={Marius Mosbach and Maksym Andriushchenko and Dietrich Klakow},
booktitle={International Conference on Learning Representations},
year={2021}
}

@techreport{krizhevsky2009learning,
  author={Krizhevsky Alex},
    title = {Learning multiple layers of features from tiny images},
    institution = {Toronto, ON, Canada},
    year = {2009}
}

@inproceedings{haarnoja2018soft,
  title={Soft actor-critic: Off-policy maximum entropy deep reinforcement learning with a stochastic actor},
  author={Haarnoja, Tuomas and Zhou, Aurick and Abbeel, Pieter and Levine, Sergey},
  booktitle={International conference on machine learning},
  pages={1861--1870},
  year={2018}
}

@inproceedings{
reddi2018convergence,
title={On the Convergence of Adam and Beyond},
author={Sashank J. Reddi and Satyen Kale and Sanjiv Kumar},
booktitle={International Conference on Learning Representations},
year={2018}
}

@article{cao2024grams,
  title={Grams: Gradient descent with adaptive momentum scaling},
  author={Cao, Yang and Li, Xiaoyu and Song, Zhao},
  journal={arXiv preprint arXiv:2412.17107},
  year={2024}
}

@article{dong2024convergence,
  title={Convergence rate analysis of lion},
  author={Dong, Yiming and Li, Huan and Lin, Zhouchen},
  journal={arXiv preprint arXiv:2411.07724},
  year={2024}
}

@inproceedings{aitchison2023atari,
  title={Atari-5: Distilling the arcade learning environment down to five games},
  author={Aitchison, Matthew and Sweetser, Penny and Hutter, Marcus},
  booktitle={International Conference on Machine Learning},
  pages={421--438},
  year={2023}
}

@article{bellemare2013arcade,
  title={The arcade learning environment: An evaluation platform for general agents},
  author={Bellemare, Marc G and Naddaf, Yavar and Veness, Joel and Bowling, Michael},
  journal={Journal of artificial intelligence research},
  volume={47},
  pages={253--279},
  year={2013}
}

@article{schulman2017proximalpolicyoptimizationalgorithms,
  title={Proximal policy optimization algorithms},
  author={Schulman, John and Wolski, Filip and Dhariwal, Prafulla and Radford, Alec and Klimov, Oleg},
  journal={arXiv preprint arXiv:1707.06347},
  year={2017}
}

@article{huang2022cleanrl,
  title={Cleanrl: High-quality single-file implementations of deep reinforcement learning algorithms},
  author={Shengyi Huang and Rousslan Fernand Julien Dossa and Chang Ye and Jeff Braga},
  journal={Journal of Machine Learning Research},
  volume={23},
  number={274},
  pages={1--18},
  year={2022}
}

@inproceedings{weng2022envpoolhighlyparallelreinforcement,
  title={Envpool: A highly parallel reinforcement learning environment execution engine},
  author={Weng, Jiayi and Lin, Min and Huang, Shengyi and Liu, Bo and Makoviichuk, Denys and Makoviychuk, Viktor and Liu, Zichen and Song, Yufan and Luo, Ting and Jiang, Yukun and others},
  booktitle={Advances in Neural Information Processing Systems},
  volume={35},
  pages={22409--22421},
  year={2022}
}

@article{machado2018revisiting,
  title={Revisiting the arcade learning environment: Evaluation protocols and open problems for general agents},
  author={Machado, Marlos C and Bellemare, Marc G and Talvitie, Erik and Veness, Joel and Hausknecht, Matthew and Bowling, Michael},
  journal={Journal of Artificial Intelligence Research},
  volume={61},
  pages={523--562},
  year={2018}
}

@article{duchi2011adaptive,
  title={Adaptive subgradient methods for online learning and stochastic optimization.},
  author={Duchi, John and Hazan, Elad and Singer, Yoram},
  journal={Journal of machine learning research},
  volume={12},
  number={7},
  year={2011}
}

@article{zeiler2012adadeltaadaptivelearningrate,
  title={ADADELTA: An Adaptive Learning Rate Method}, 
  author={Matthew D. Zeiler},
  journal={arXiv preprint arXiv:1212.5701},
  year={2012}
}

@article{mandt2017stochastic,
  title={Stochastic gradient descent as approximate bayesian inference},
  author={Mandt, Stephan and Hoffman, Matthew D and Blei, David M},
  journal={Journal of Machine Learning Research},
  volume={18},
  number={134},
  pages={1--35},
  year={2017}
}

@article{song2022learning,
  title={Learning from noisy labels with deep neural networks: A survey},
  author={Song, Hwanjun and Kim, Minseok and Park, Dongmin and Shin, Yooju and Lee, Jae-Gil},
  journal={IEEE transactions on neural networks and learning systems},
  volume={34},
  number={11},
  pages={8135--8153},
  year={2022}
}

@inproceedings{zhang2017understanding,
  title     = {Understanding Deep Learning Requires Rethinking Generalization},
  author    = {Zhang, Chiyuan and Bengio, Samy and Hardt, Moritz and Recht, Benjamin and Vinyals, Oriol},
  booktitle = {International Conference on Learning Representations},
  year      = {2017}
}

@inproceedings{henderson2018matters,
  title     = {Deep Reinforcement Learning that Matters},
  author    = {Henderson, Peter and Islam, Riashat and Bachman, Philip and Pineau, Joelle and Precup, Doina and Meger, David},
  booktitle = {Proceedings of the AAAI Conference on Artificial Intelligence (AAAI)},
  volume    = {32},
  year      = {2018}
}

@article{lan2025learning,
    title={Learning to Optimize for Reinforcement Learning},
    author={Lan, Qingfeng and Mahmood, A. Rupam and YAN, Shuicheng and Xu, Zhongwen},
    journal={Reinforcement Learning Journal},
    volume={2},
    pages={481--497},
    year={2025}
}

@inproceedings{
loshchilov2018decoupled,
title={Decoupled Weight Decay Regularization},
author={Ilya Loshchilov and Frank Hutter},
booktitle={International Conference on Learning Representations},
year={2019}
}

@article{lyle2024normalization,
  title={Normalization and effective learning rates in reinforcement learning},
  author={Lyle, Clare and Zheng, Zeyu and Khetarpal, Khimya and Martens, James and van Hasselt, Hado P and Pascanu, Razvan and Dabney, Will},
  journal={Advances in Neural Information Processing Systems},
  volume={37},
  pages={106440--106473},
  year={2024}
}
\bibliographystyle{iclr2025_conference}

\appendix

\section{Anolog pseudo code}
\label{app:anolog-pseudocode}

\begin{algorithm}[H]
\caption{Anolog}
\Input{Initial parameters $x_1 \in \mathbb{R}^d$, learning rate $\eta_k$, decay rate $\beta_2 \in [0,1)$, $\epsilon > 0$}
\BlankLine
Initialize $m_0 = 0$, $v_0 = 0$ \\
\For{$k = 1$ \KwTo $K$}{
Compute gradient $g_k = \nabla \ell(x_k)$ \\
$\beta_1 = 1-\frac{1}{\log(k+2)}$\\
$m_k = \beta_1 m_{k-1} + (1 - \beta_1) g_k$ \\
$v_k = \beta_2v_{k-1} - (1 - \beta_2) \cdot \text{sign}(v_{k-1}-g_k^2) \cdot g_k^2$ \\
$\hat{v_k} = \frac{v_k}{1-\beta_2^k}$\\
$x_{k+1} = x_k - \frac{\eta_k}{\sqrt{\hat{v_k}} + \epsilon} \cdot |g_k| \cdot \text{sign}(m_k) - \eta_k\lambda x_k$
}
\label{alg:anolog}
\end{algorithm}

\section{Hyperparameter Tuning Protocol}
\label{app:tuning_experiments}

To ensure a fair comparison, we conducted exhaustive grid searches over learning rates and momentum parameters on lightweight proxy tasks representative of each domain, using 5 independent seeds for each hyperparameter combination.
In total, the campaign involved \textbf{2115} independent training runs: roughly 35 hours per optimizer when searching across three hyperparameters, and about 12 hours for optimizers with only two (e.g., Anolog, RMSprop). For Adan, which introduces a third momentum coefficient, we maintained a uniform computational budget by varying only $(\beta_{1}, \beta_{3})$, that controlling the first and second moment estimates, while fixing the Nesterov term at its default value, $\beta_{2}=0.92$. For Anolog and RMSprop, we tuned only the learning rate and the variance term decay parameter ($\beta_2$ for Anolog and $\alpha$ for RMSprop). We then selected the configuration achieving the highest validation accuracy per seed, with final hyperparameters for all optimizers summarized in Table~\ref{tab:optimizers-hyperparameters}.

\paragraph{Computer-vision proxy.}
From CIFAR-10 \citep{krizhevsky2009learning}, we drew a balanced subset of $10{,}000$ training images and $2{,}000$ test images.  
We applied the identical augmentation pipeline used in Section~\ref{sec:experiments} (random cropping, horizontal flips, and Cutout).  
Each hyperparameter configuration was trained for 20 epochs under five independent seeds with a ResNet-18 backbone\citep{he2016deep}.

\begin{figure}[H]
    \centering
    \includegraphics[width=0.75\textwidth]{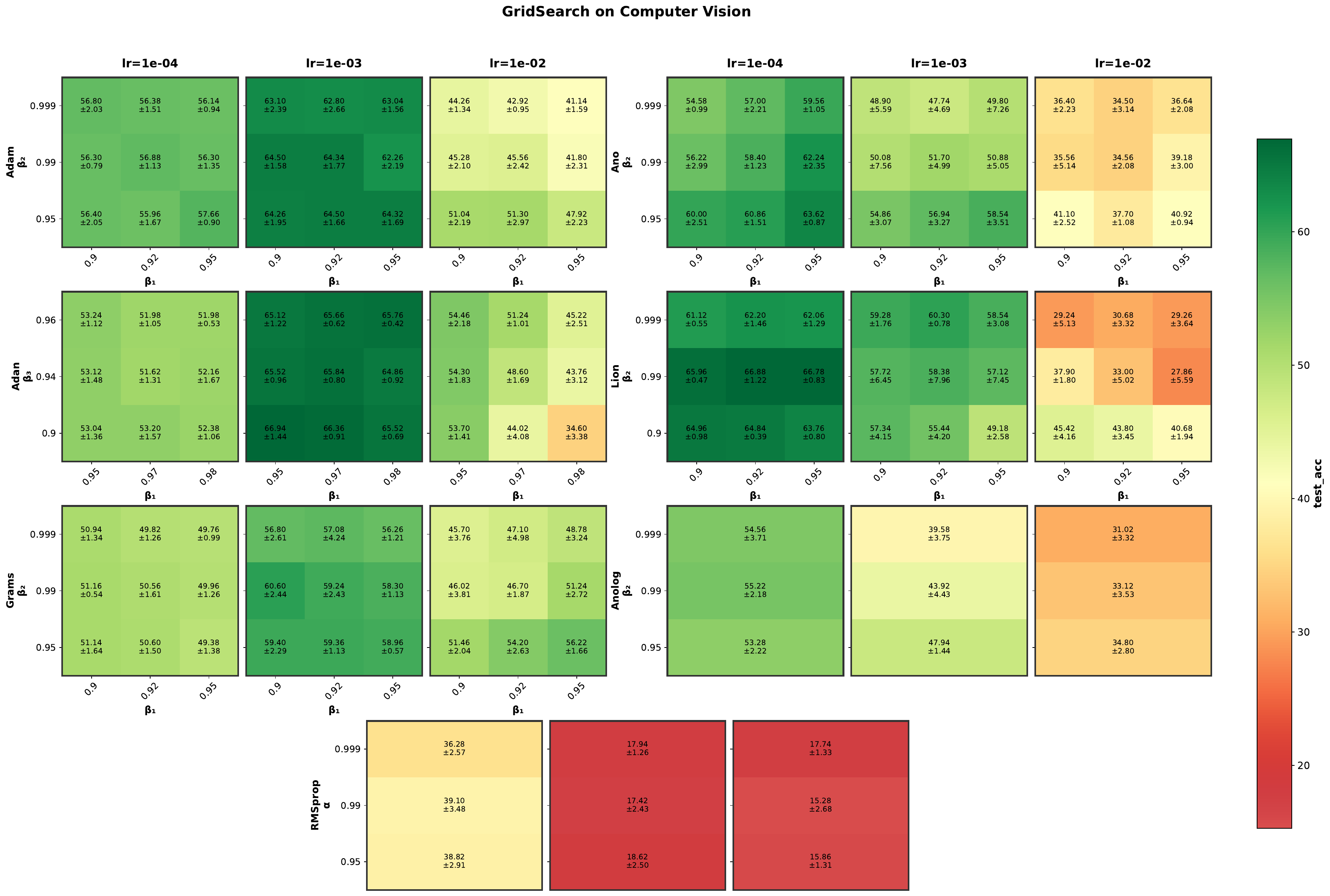}
    \caption{Grid search on the CIFAR-10 proxy (ResNet-18) for the optimizers.}
    \label{fig:cifar10_proxy}
\end{figure}

\paragraph{NLP proxy.}
We fine-tune Bert \citep{devlin2019bert} on the MRPC benchmark\citep{wang2019gluemultitaskbenchmarkanalysis}.
Although MRPC is relatively small and noisy, this characteristic amplifies the impact of optimizer hyperparameters, making it easier to reveal differences in optimization behavior that may be less pronounced on larger, more stable datasets. At the same time, its modest size keeps the experiments computationally efficient while preserving representative fine-tuning dynamics of GLUE tasks. A preliminary sweep indicated that learning rates outside the range $[1!\times!10^{-5},,7!\times!10^{-5}]$ consistently led to poor accuracy; subsequent grids therefore focus on a narrower range centered at $2!\times!10^{-5}$. 

\begin{figure}[H]
    \centering
    \includegraphics[width=0.75\textwidth]{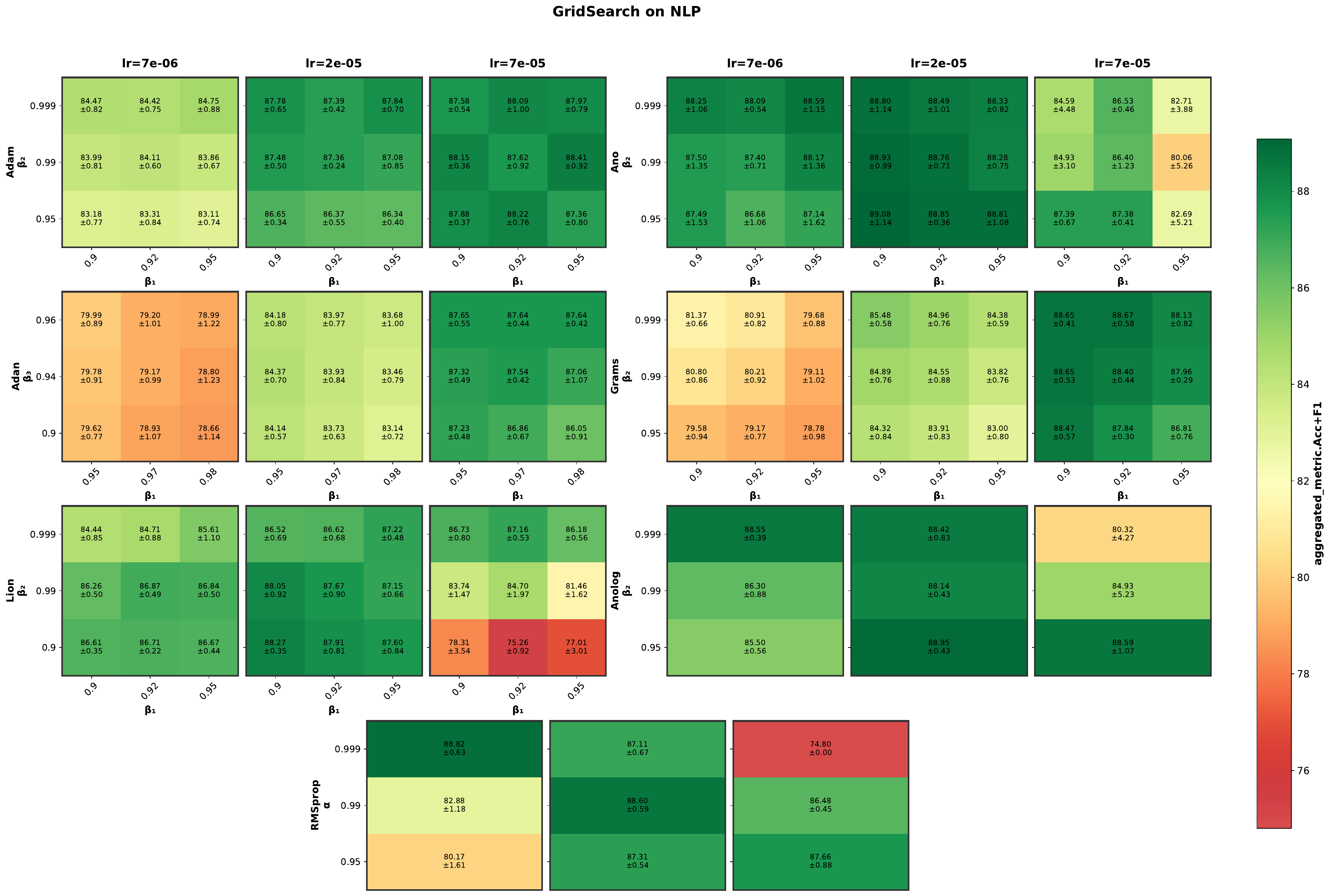}
    \caption{Grid-search results on MRPC used as an NLP proxy.}
    \label{fig:mrpc_proxy}
\end{figure}

\paragraph{Deep-RL proxy.}
For Deep Reinforcement Learning, we train a SAC agent on the MuJoCo \textit{HalfCheetah-v5} environment for 100k steps, given time constraints\citep{todorov2012mujoco, haarnoja2018soft}. This setup is primarily intended to reveal the impact of different hyperparameters, especially the momentum coefficients $\beta$, though we note that the shorter horizon may favor more aggressive learning rates. To address this limitation, the main text reports the best performance between default and tuned hyperparameters for each optimizer.

\begin{figure}[H]
    \centering
    \includegraphics[width=0.75\textwidth]{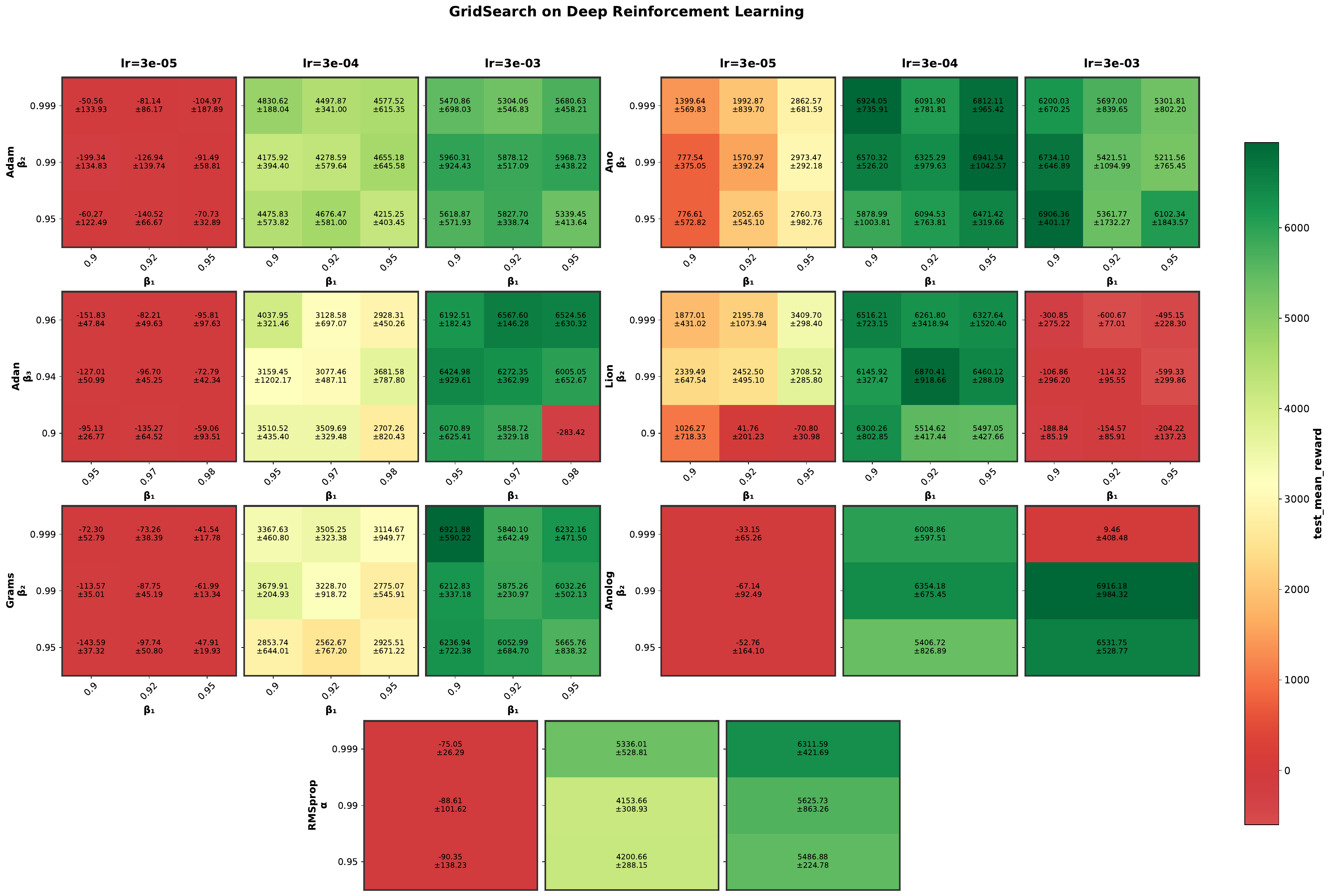}
    \caption{Grid-search on HalfCheetah 100k-steps training used as a DRL proxy.}
    \label{fig:halfcheetah_proxy}
\end{figure}

\section{Convergence Proof for \textsc{Ano}}\label{app:convergence_proof}

\subsection{Algorithmic update}

For each coordinate $i\in[d]$ the optimizer maintains a first-order momentum
$m_{k,i}$ and a second moment $v_{k,i}$ updated as below:
\begin{equation}\label{eq:yogi-updates}
m_{k,i} = \beta_{1,k} m_{k-1,i} + (1-\beta_{1,k}) g_{k,i},\qquad
v_{k,i} = \beta_2 v_{k-1,i}
          - (1-\beta_{2})\,\operatorname{sign}\!\bigl(v_{k-1,i}-g_{k,i}^{2}\bigr)\,g_{k,i}^{2},
\end{equation}
with $\beta_{1,k}=1-\frac{1}{\sqrt{k+1}}$ for \(k\ge1\); $\beta_{2}\in[0.5,1)$ and $m_{0,i}=v_{0,i}=0$.
The parameter vector is then updated by
\begin{equation}\label{eq:ano-step}
x_{k+1,i}
  = x_{k,i}
    - \frac{\eta_{k}}{\sqrt{v_{k-1,i}}+\varepsilon}\,
      |g_{k,i}|\,\operatorname{sign}(m_{k,i}),
\end{equation}

where $\varepsilon>0$ is a fixed constant.
\begin{equation}\label{eq:eta-schedule}
\eta_{k} = \frac{\eta}{(k+2)^{3/4}},\qquad k = 0,1,\dotsc
\end{equation}
We write $\E_{k-1}[\cdot] := \E[\cdot \mid \mathcal{F}_{k-1}]$ for the conditional expectation given the filtration $\mathcal{F}_{k-1}$.

\subsection{Standing assumptions}
\label{ass:all}
\begin{assumption}[Smoothness]\label{ass:smoothness}
The objective function $f : \mathbb{R}^d \to \mathbb{R}$ is differentiable and $L$-smooth; that is, for all $x, y \in \mathbb{R}^d$,
\[
\|\nabla f(x) - \nabla f(y)\| \leq L \|x - y\|.
\]
\end{assumption}

\begin{assumption}[Lower boundedness]\label{ass:lower-bounded}
The function $f$ is bounded from below: there exists $f^\star > -\infty$ such that $f(x) \geq f^\star$ for all $x \in \mathbb{R}^d$.
\end{assumption}

\begin{assumption}[Bounded gradients]\label{ass:bounded-grad}
There exists a constant $G > 0$ such that $|\nabla_i f(x_k)| \le G$ for all iterates $x_k$ and all coordinates $i$.
\end{assumption}

\begin{assumption}[Unbiased stochastic gradients]\label{ass:unbiased}
At each iteration $k$, we observe a stochastic gradient $g_k \in \mathbb{R}^d$ satisfying $\E[g_{k,i} \mid \mathcal{F}_{k-1}] = \nabla_i f(x_k)$ for all $i \in [d]$.
\end{assumption}

\begin{assumption}[Bounded variance]\label{ass:bounded-var}
There exists $\sigma > 0$ such that for all $i \in [d]$ and $k \ge 1$,
\[
\E[(g_{k,i} - \nabla_i f(x_k))^2 \mid \mathcal{F}_{k-1}] \leq \sigma^2.
\]
\end{assumption}

\subsection{Preliminary Lemma}
\paragraph{Local assumption.} 
For this lemma only, we impose an additional assumption solely to simplify the analysis and to obtain a pointwise bound on $v_{k,i}$. 
Specifically, we assume $g_{k,i} \leq \tilde{G}$ for all $k,i$. 
This assumption is not used anywhere else in the paper and plays no role in the convergence results.

\begin{lemma}[Bounds on $v_k$]\label{lem:vk-bounds}
Fix any coordinate \( i \in [d] \) and assume \( v_0 = 0 \) and \( \beta_2 \in [\tfrac{1}{2},1) \).  
Then for every \( k \ge 0 \),
\[
0 \;\le\; v_{k,i} \;\le\; \tilde{G}^2.
\]
\end{lemma}

\begin{proof}
The update is
\[
v_{k,i} = \beta_2 v_{k-1,i}
- (1-\beta_2)\,\sign(v_{k-1,i}-g_{k,i}^{2})\,g_{k,i}^{2}.
\]

\paragraph{Upper bound.}  
If \( \sign(v_{k-1,i}-g_{k,i}^2)=-1 \), then  
\[
v_{k,i} = \beta_2 v_{k-1,i} + (1-\beta_2)g_{k,i}^2
\le \tilde{G}^2
\]
If \( \sign(v_{k-1,i}-g_{k,i}^2)=1 \), then  
\[
v_{k,i} = \beta_2 v_{k-1,i} - (1-\beta_2)g_{k,i}^2 
\le \beta_2 v_{k-1,i} \le v_{k-1,i}.
\]
Starting from \( v_0=0 \), induction gives \( v_{k,i}\le G^2 \) for all \( k \).

\paragraph{Lower bound.}  
If \( \sign(v_{k-1,i}-g_{k,i}^2)=-1 \), then \( v_{k,i} \ge 0 \) since it is a convex combination of nonnegative terms.  
If \( \sign(v_{k-1,i}-g_{k,i}^2)=1 \), then
\[
v_{k,i} = \beta_2 v_{k-1,i} - (1-\beta_2)g_{k,i}^2
\ge (2\beta_2-1)g_{k,i}^2 \;\ge\; 0
\]
because \( \beta_2 \ge \tfrac12 \) and \( g_{k,i}^2 \ge 0 \).  
Thus \( v_{k,i} \ge 0 \) for all \( k \).
\end{proof}

\subsection{Auxiliary quantities}

Define the two per-iteration scalars
\begin{equation}\label{eq:AkBk-def}
A_{k} := \sum_{i=1}^{d}
         \frac{\nabla_{i}f(x_{k})\,|g_{k,i}|\,\operatorname{sign}(m_{k,i})}
              {\sqrt{v_{k-1,i}}+\varepsilon},
\qquad
B_{k} := \frac{L}{2}\sum_{i=1}^{d}
         \frac{g_{k,i}^{2}}{(\sqrt{v_{k-1,i}}+\varepsilon)^{2}}.
\end{equation}
These two terms govern the decrease of the objective:
\begin{equation}\label{eq:smooth-descent}
f(x_{k+1}) \le f(x_{k}) - \eta_{k} A_{k} + \eta_{k}^{2} B_{k}.
\end{equation}

\subsection{Lemma for Probability Sign-Mismatch}

\begin{lemma}[Sign-Mismatch Probability for Ano]\label{lem:sign-lemma}
Fix any coordinate $i\in[d]$, under Assumptions \ref{ass:smoothness}–\ref{ass:bounded-var} and following update rules \eqref{eq:ano-step}-\eqref{eq:yogi-updates}, with $\beta_{1,k} = 1 - \frac{1}{\sqrt{k+1}}$ and $\eta_k = (k+2)^{-3/4}$.
Then for every $k\ge1$,
\[
\mathbb{P}(\sign(m_{k,i}) \neq \sign(\nabla_i f(x_k)) \leq \frac{C_m^2}{|\nabla_i f(x_k)|^2\sqrt{k+1}}
\]
with \(C_m \coloneqq \sqrt{2(C_{\Delta}^2+\sigma^2)}\) and \(C_{\Delta}^2 = \frac{L^2d(\sigma^2+G^2)}{\varepsilon^2}\)
\end{lemma}

\begin{proof}
We fix a coordinate \( i \in [d] \). Define the per-coordinate momentum error as
\[
e_{k,i} := m_{k,i} - \nabla_i f(x_k).
\]
Subtracting $\nabla_if(x_k)$ from the update rule of $m_{k,i}$ yields:
\begin{align*}
e_{k,i} &= \beta_{1,k} m_{k-1,i} + (1 - \beta_{1,k}) g_{k,i} - \nabla_i f(x_k) \\
&= \beta_{1,k} (m_{k-1,i} - \nabla_i f(x_{k-1})) + \beta_{1,k} (\nabla_i f(x_{k-1}) - \nabla_i f(x_k)) + (1 - \beta_{1,k})(g_{k,i} - \nabla_i f(x_k))
\end{align*}
Define \( \Delta_{k,i} := \nabla_i f(x_{k-1}) - \nabla_i f(x_k) \) as the gradient variation, and \( \xi_{k,i} := g_{k,i} - \nabla_i f(x_k) \) as the stochastic noise for coordinate \( i \).
\[
e_{k,i}= \beta_{1,k} e_{k-1,i} + \beta_{1,k} \Delta_{k,i} + (1 - \beta_{1,k})\xi_{k,i},
\]
Conditionally on \(\mathcal F_{k-1}\), define \( V_{k,i} := \E_{k-1}[e_{k,i}^2] \) :
\[
V_{k,i} = \E_{k-1}\left[ \left( \beta_{1,k} e_{k-1,i} + \beta_{1,k} \Delta_{k,i} + (1 - \beta_{1,k})\xi_{k,i} \right)^2 \right].
\]
Since $\E_{k-1}[\xi_{k,i}] = 0$ (Ass. \ref{ass:unbiased}), all mixed terms involving $\xi_{k,i}$ vanish after taking $\E_{k,i}$. Thus,
\[
V_{k,i} = \E_{k-1}\left[ \left( \beta_{1,k} e_{k-1,i} + \beta_{1,k} \Delta_{k,i} \right)^2 \right] + (1 - \beta_{1,k})^2 \E_{k-1}[\xi_{k,i}^2].
\]

We now apply Young's inequality for scalars: \( (a + b)^2 \leq (1 + \delta) a^2 + (1 + 1/\delta) b^2 \) for any \( \delta > 0 \). We set \( \delta = \frac{1}{\sqrt{k+1}}\) (this minimizes the resulting upper bound). Applying this gives:
\[
\E_{k-1}\left[ \left( \beta_{1,k} e_{k-1,i} + \beta_{1,k} \Delta_{k,i} \right)^2 \right] \leq \left(1 + \frac{1}{\sqrt{k+1}}\right)\beta_{1,k}^2 \E_{k-1}[e_{k-1,i}^2] + (1 +\sqrt{k+1})\beta_{1,k}^2 \E_{k-1}[\Delta_{k,i}^2].
\]
Therefore,
\[
V_{k,i} \leq \left(1 + \frac{1}{\sqrt{k+1}}\right)\beta_{1,k}^2 V_{k-1,i} + (1 + \sqrt{k+1})\beta_{1,k}^2 \E_{k-1}[\Delta_{k,i}^2] + (1 - \beta_{1,k})^2 \sigma^2.
\]

To bound \( \E_{k-1}[\Delta_{k,i}^2] \), we use the L-smoothness of \( f \) (Assumption \ref{ass:smoothness}). We have:
\[
|\Delta_{k,i}| = |\nabla_i f(x_{k-1}) - \nabla_i f(x_k)| \leq \|\nabla f(x_{k-1}) - \nabla f(x_k)\|_2 \leq L \|x_k - x_{k-1}\|_2.
\]
Hence,
\[
\E_{k-1}[\Delta_{k,i}^2] \leq L^2 \E_{k-1}[\|x_k - x_{k-1}\|_2^2].
\]

Now we bound the step size:
\begin{align*}
\E_{k-1}[\|x_k - x_{k-1}\|_2^2] &= \E_{k-1}\left[ \sum_{j=1}^d \left( \frac{\eta_{k-1}}{\sqrt{v_{k-1,j}} + \varepsilon} |g_{k-1,j}| \right)^2 \right] \\
&\leq \eta_{k-1}^2 \sum_{j=1}^d \E_{k-1}\left[ \frac{g_{k-1,j}^2}{\varepsilon^2} \right] \\
&\leq \frac{\eta_{k-1}^2}{\varepsilon^2} \sum_{j=1}^d \E_{k-1}[g_{k-1,j}^2].
\end{align*}
By Assumptions \ref{ass:bounded-grad}-\ref{ass:bounded-var}, we have \( \E_{k-1}[g_{k-1,j}^2] \leq \sigma^2 + G^2 \), where \( \sigma^2 \) is the variance bound and \( G \) is an upper bound on the gradient norm. Thus,
\[
\E_{k-1}[\|x_k - x_{k-1}\|_2^2] \leq \frac{d \eta_{k-1}^2 (\sigma^2 + G^2)}{\varepsilon^2}.
\]
Let \( C_\Delta^2 := \frac{L^2 d (\sigma^2 + G^2)}{\varepsilon^2} \). Then, we obtain:
\[
\E_{k-1}[\Delta_{k,i}^2] \leq C_\Delta^2 \eta_{k-1}^2.
\]

Putting everything together, we get the recurrence:
\[
V_{k,i} \leq \left(1 + \frac{1}{\sqrt{k+1}}\right)\beta_{1,k}^2 V_{k-1,i} + (1 + \sqrt{k+1})\beta_{1,k}^2 C_\Delta^2 \eta_{k-1}^2 + (1 - \beta_{1,k})^2 \sigma^2.
\]
Set $\beta_{1, k} = 1 - \frac{1}{\sqrt{k+1}}$ and $\eta_k = 1/(k+1)^{3/4}$, to simplify we denote $l = k+1$.
\begin{align*}
    V_{k,i} &\leq (1 + \frac{1}{\sqrt{l}})(1 - \frac{1}{\sqrt{l}})^2 V_{k-1,i} + (1 + \sqrt{l})(1 - \frac{1}{\sqrt{l}})^2 C_\Delta^2 \frac{1}{l^{3/2}} + \frac{\sigma^2}{k+1}\\
    V_{k,i} &\leq (1-\frac{1}{\sqrt{l}}-\frac{1}{l}+\frac{1}{l^{3/2}}) V_{k-1,i} + \frac{1}{l} \left[ (1 + \frac{1}{\sqrt{l}})(1 - \frac{1}{\sqrt{l}})^2 C_\Delta^2 + \sigma^2 \right]\\
\end{align*}

Hence \(a_k \coloneqq (1-\frac{1}{\sqrt{l}}-\frac{1}{l}+\frac{1}{l^{3/2}})\) and \( B_k \coloneqq \frac{1}{l} \left[ (1 + \frac{1}{\sqrt{l}})(1 - \frac{1}{\sqrt{l}})^2 C_\Delta^2 + \sigma^2 \right]\), with $a_k \le 1 - \frac{1}{\sqrt{l}}$ and $B_k = \mathcal{O}(l^{-1})$.So, we can express the inequality on this way 
\[
V_{k,i} \leq a_k V_{k-1,i} + B_k
\]

Simplifying the coefficients.
For $k\ge1$,
\[
  a_k
  \;=\;
  1-\frac1{\sqrt{l}}-\frac1{l}+\frac1{l^{3/2}}
  \;\;\le\;\;
  1-\frac1{\sqrt{l}},
  \qquad
  (1+l^{-1/2})(1-l^{-1/2})^2
  \;\le\;
  1,
\]
so that
\begin{equation}
  a_k \le 1-\frac1{\sqrt{l}},
  \qquad
  B_k \le \frac{C_B}{l},
  \quad
  C_B := C_\Delta^{2}+\sigma^{2}.
\end{equation}

We want to simplify this bound by having a bound with the form $V_{k,i} \le \mathcal{O}\left( \frac{1}{\sqrt{l}}\right)$, to proceed, we will prove by induction that $V_{k,i} = O(1/\sqrt{l})$. Specifically, we posit that there exists a constant $M$ such that for all sufficiently large $k$, $V_{k,i} \leq \frac{M}{\sqrt{l}}$.

\paragraph{Base Case} We can choose $M$ large enough such that the hypothesis holds for some initial $k_0 \geq 1$.

\paragraph{Inductive Step} Assume that for some $k > k_0$, the hypothesis $V_{k-1} \leq \frac{M}{\sqrt{k-1}}$ holds. We must show that $V_{k,i} \leq \frac{M}{\sqrt{l}}$. From our simplified recurrence, we have:
\[
V_{k,i} \leq \left(1 - \frac{1}{\sqrt{l}}\right) \frac{M}{\sqrt{l-1}} + \frac{C_B}{l}
\]
The induction holds if we can prove:
\[
\left(1 - \frac{1}{\sqrt{l}}\right) \frac{M}{\sqrt{l-1}} + \frac{C_B}{l} \leq \frac{M}{\sqrt{l}}
\]
Rearranging the terms, this is equivalent to showing:
\[
\frac{C_B}{l} \leq M \left( \frac{1}{\sqrt{l}} - \frac{1}{\sqrt{l-1}} + \frac{1}{\sqrt{l}\sqrt{l-1}} \right)
\]
To analyze the right-hand side (RHS) for large $k$, we use a Taylor expansion for the term $(l-1)^{-1/2}$:
\[
\frac{1}{\sqrt{l-1}} = (l-1)^{-1/2} = l^{-1/2} \left(1 - \frac{1}{l}\right)^{-1/2}
\]
Using the expansion $(1-x)^{-1/2} = 1 + \frac{x}{2} + O(x^2)$, we get:
\[
\frac{1}{\sqrt{l-1}} = \frac{1}{\sqrt{l}} \left(1 + \frac{1}{2l} + O\left(\frac{1}{l^2}\right)\right) = \frac{1}{\sqrt{l}} + \frac{1}{2l^{3/2}} + O\left(\frac{1}{l^{5/2}}\right)
\]
Substituting this into the parenthesis on the RHS of our inequality, the term becomes:
\begin{align*}
    & \frac{1}{\sqrt{l}} - \left(\frac{1}{\sqrt{l}} + \frac{1}{2l^{3/2}}\right) + \frac{1}{l\sqrt{1-1/l}} + O\left(\frac{1}{l^{5/2}}\right) \\
    = & -\frac{1}{2l^{3/2}} + \frac{1}{l}\left(1 + \frac{1}{2l} + O\left(\frac{1}{l^2}\right)\right) + O\left(\frac{1}{l^{5/2}}\right) \\
    = & \frac{1}{l} - \frac{1}{2l^{3/2}} + O\left(\frac{1}{l^2}\right)
\end{align*}
The full inequality we need to satisfy is therefore:
\[
\frac{C_B}{l} \leq M \left( \frac{1}{l} - \frac{1}{2l^{3/2}} + O\left(\frac{1}{l^2}\right) \right)
\]
Multiplying through by $l$, the condition becomes:
\[
C_B \leq M \left( 1 - \frac{1}{2\sqrt{l}} + O\left(\frac{1}{l}\right) \right)
\]
This inequality shows why the induction works. For any choice of constant $M > C_B$, we can find a sufficiently large $k_0$ such that for all $k \geq k_0$, the inequality holds. This completes the induction, establishing that $V_{k,i} \leq \frac{M}{\sqrt{l}}$. So, for a sufficiently large $k_0$ such that for all $k \geq k_0$, we have :
\begin{align*}
    V_{k,i} &\leq \frac{2C_B}{\sqrt{l}}\\
    \E_{k-1}[(m_{k,i} - \nabla_i f(x_k))^2] &\leq \frac{2(C_{\Delta}^2+\sigma^2)}{\sqrt{l}}\\
    \E_{k-1}[|m_{k,i} - \nabla_i f(x_k)|^2] &\leq \frac{2(C_{\Delta}^2+\sigma^2)}{\sqrt{l}}\\
 \end{align*}
We recall that $l=k+1$, so 
\[
\E_{k-1}[|m_{k,i} - \nabla_i f(x_k)|^2] \leq \frac{2(C_{\Delta}^2+\sigma^2)}{\sqrt{k+1}}\\
\]
\textbf{From moment bound to probability bound}
We bound the probability of a momentum--gradient sign mismatch. If $\sign(m_{k,i}) \neq \sign(\nabla_i f(x_k))$ and $\nabla_i f(x_k)\neq 0$, then $|m_{k,i}-\nabla_if(x_k)| \geq |\nabla_i f(x_k)|$. Hence, for any $k\ge1$,
\[
\mathbb{P}_{k-1}\big(\sign(m_{k,i}) \neq \sign(\nabla_i f(x_k))\big) \leq \mathbb{P}\big(|m_{k,i}-\nabla_if(x_k)| \geq |\nabla_i f(x_k)|\big)
\]
We apply Chebyshev's inequality to the right-hand side:
\[
\mathbb{P}_{k-1}\big(|m_{k,i}-\nabla_if(x_k)| \geq |\nabla_i f(x_k)|\big) \leq \frac{\mathbb{E}_{k-1}\big[|m_{k,i}-\nabla_if(x_k)|^2\big]}{|\nabla_i f(x_k)|^2}
\]
Using the previously established second-moment bound, $\mathbb{E}_{k-1}[|m_{k,i} - \nabla_i f(x_k)|^2] \leq \frac{2(C_{\Delta}^2+\sigma^2)}{\sqrt{k+1}}$ :
\[
\mathbb{P}_{k-1}\big(\sign(m_{k,i}) \neq \sign(\nabla_i f(x_k))\big) \leq \frac{C_m^2}{|\nabla_i f(x_k)|^2\sqrt{k+1}}
\]
with \( C_m = \sqrt{2(C_{\Delta}^2 + \sigma^2)}\)
\end{proof}

\subsection{Bound on A}
\begin{lemma}[Lower bound on the expected update magnitude]\label{lem:Ak-lower}
Assume all the condition (cf \ref{ass:all}.  
Recall
\[
A_k
\;=\;
\sum_{i=1}^{d}
\frac{\nabla_{i}f(x_{k})\,|g_{k,i}|\,\operatorname{sign}(m_{k,i})}
     {\sqrt{v_{k-1,i}}+\varepsilon},
\qquad
\E_{k-1}[\cdot]\;=\;\E[\cdot \mid \mathcal{F}_{k-1}].
\]
Let
\[
C_m = \sqrt{2(C_{\Delta}^2+\sigma^2)}, \quad C_v = \frac{2d\sqrt{\sigma^2+G^2}}{\varepsilon}
\]
Then, for every iteration \(k\ge k_0\),
\[
\E[A_k]
\geq
\frac{\E[\|\nabla f(x_k)\|^2_2]}{\tilde{G} + \varepsilon}
-
\frac{C_v C_m}{(k+1)^{1/4}},
\]
\end{lemma}

\begin{proof}
We begin by recalling the definition of \( A_k \) and factoring out constants that do not depend on \( g_{k,i} \):
\[
\E_{k-1}[A_k]
=
\sum_{i=1}^{d}
\frac{\nabla_{i}f(x_{k})}{\sqrt{v_{k-1,i}}+\varepsilon}
\cdot
\E_{k-1}\bigl[|g_{k,i}|\operatorname{sign}(m_{k,i})\bigr].
\]
Our goal is to lower bound the term \(\E_{k-1}\bigl[|g_{k,i}|\operatorname{sign}(m_{k,i})\bigr]\).

We first expand this term using the identity:
\[
\operatorname{sign}(m_{k,i}) = \operatorname{sign}(\nabla_i f(x_k)) \cdot \left(1 - 2 \cdot \mathbb{I}\left[\operatorname{sign}(\nabla_i f(x_k)) \neq \operatorname{sign}(m_{k,i})\right]\right),
\]
Let \( \chi_{k,i} := \mathbb{I}\left[\operatorname{sign}(\nabla_i f(x_k)) \neq \operatorname{sign}(m_{k,i})\right] \).
, we get:
\begin{align*}
\E_{k-1}\bigl[|g_{k,i}|\operatorname{sign}(m_{k,i})\bigr]
&= \operatorname{sign}(\nabla_i f(x_k)) \cdot \E_{k-1}\left[|g_{k,i}|\right]
\\
&\quad \quad- 2 \cdot \E_{k-1}\left[|g_{k,i}| \cdot \operatorname{sign}(\nabla_i f(x_k)) \cdot \chi_{k,i}\right]
\end{align*}
where the second line follows from linearity of expectation.

We now bound the first term using Jensen
\[
\operatorname{sign}(\nabla_i f(x_k)) \cdot \E_{k-1}[|g_{k,i}|]
\geq
\operatorname{sign}(\nabla_i f(x_k)) \cdot \left|\E_{k-1}[g_{k,i}]\right|
= \nabla_i f(x_k),
\]
where we used the assumption that \(\E_{k-1}[g_{k,i}] = \nabla_i f(x_k)\).

We bound the second term by combining \(\operatorname{sign}(\nabla_i f(x_k)) \leq 1\) and Cauchy-Schwarz:
\begin{align*}
    \E_{k-1}\left[|g_{k,i}| \cdot \operatorname{sign}(\nabla_i f(x_k))\cdot \chi_{k,i}\right] &\leq \sqrt{\E_{k-1}\left[g_{k,i}^2\right]} \cdot \sqrt{\E_{k-1}\left[\chi_{k,i}^2\right]}\\
\end{align*}
By the variance definition, we know that \( \operatorname{Var}(X) = \E[X^2] - (\E[X])^2 \), so we have
\[
 \E_{k-1}\left[g_{k,i}^2\right] = \operatorname{Var}_{k-1}(g_{k,i}) + \left(\E_{k-1}[g_{k,i}]\right)^2
\]
By combining Assumptions (\ref{ass:unbiased}, \ref{ass:bounded-grad} and \ref{ass:bounded-var}), and \( \chi_{k,i}^2 = \chi_{k,i}\), we got :
\[
\E_{k-1}\left[|g_{k,i}| \cdot \operatorname{sign}(\nabla_i f(x_k))\cdot \chi_{k,i}\right] \leq \sqrt{\sigma^2 + G^2} \cdot \sqrt{\mathbb{P}_{k-1}\left(\chi_{k,i}\right)}
\]
By using lemma \ref{lem:sign-lemma},
\[
\E_{k-1}\left[|g_{k,i}| \cdot \operatorname{sign}(\nabla_i f(x_k)) \cdot \chi_{k,i}\right] \leq \sqrt{\sigma^2 + G^2} \cdot \frac{C_m}{|\nabla_i f(x_k)|(k+1)^{1/4}}
\]
We have so :
\begin{align*}
    \E_{k-1}\bigl[|g_{k,i}|\operatorname{sign}(m_{k,i})\bigr]
&\geq \nabla_i f(x_k)- 2 \cdot \sqrt{\sigma^2 + G^2} \cdot \frac{C_m}{|\nabla_i f(x_k)|(k+1)^{1/4}}
\end{align*}
In our main equation, we have then
\[
\E_{k-1}[A_k]
\geq
\sum_{i=1}^{d}
\frac{\nabla_{i}f(x_{k})}{\sqrt{v_{k-1,i}}+\varepsilon}
\cdot \left(
\nabla_i f(x_k)- 2 \cdot \sqrt{\sigma^2 + G^2} \cdot \frac{C_m}{|\nabla_i f(x_k)|(k+1)^{1/4}} \right)
\]
\[
\E_{k-1}[A_k]
\geq
\sum_{i=1}^{d}\frac{(\nabla_i f(x_{k}))^2}{\sqrt{v_{k-1,i}}+\varepsilon}
- \frac{1}{(k+1)^{1/4}} \cdot 2\sqrt{\sigma^2 + G^2} \;C_m \sum_{i=1}^{d} \frac{1}{\sqrt{v_{k-1,i}}+\varepsilon}
\]

Using lemma \ref{lem:vk-bounds}, we got \( 0 \leq v_{k-1,i} \leq \tilde{G}^2 \), we deduce:
\[
\E_{k-1}[A_k]
\geq
\frac{\|\nabla f(x_{k})\|^2_2}{\tilde{G}+\varepsilon}
- \frac{1}{(k+1)^{1/4}} \cdot \frac{2d\sqrt{\sigma^2 + G^2} \;C_m}{\varepsilon}
\]
Finally, letting \( C_v = \frac{2d\sqrt{\sigma^2+G^2}}{\varepsilon} \), we can write:
\[
\E_{k-1}[A_k]
\geq
\frac{\|\nabla f(x_k)\|^2_2}{\tilde{G} + \varepsilon}
-
\frac{C_v C_m}{(k+1)^{1/4}},
\]

By the total expectation law 
\[
\E[A_k]
\geq
\frac{\E[\|\nabla f(x_k)\|^2_2]}{\tilde{G} + \varepsilon}
-
\frac{C_v C_m}{(k+1)^{1/4}},
\]

which concludes the proof.
\end{proof}

\subsection{Bound on B}
\begin{lemma}
\label{lem:Bk}
Assume the standing hypotheses of the paper hold, in particular (Assumptions \ref{ass:bounded-grad}-\ref{ass:bounded-var})
\(
\E_{k-1}[g_{k,i}^2]\le G^2+\sigma^2
\), for all time-steps \(k\) and coordinates \(i\),
and let \(\varepsilon>0\).
Then, for every iteration \(k \ge 1\),
\[
  \E[B_k]
  \le\;
  \frac{L\,d\,(G^2+\sigma^2)}
       {2\,\varepsilon^{2}}
\]
\end{lemma}

\begin{proof}
Because \(v_{k-1,i}\ge 0\) and \(\varepsilon>0\), we have
\[
  (\sqrt{v_{k-1,i}}+\varepsilon)^{2}
  \;\ge\;
  \varepsilon^{2}.
\]
Together with the bound \(
\E_{k-1}[g_{k,i}^2]\le G^2+\sigma^2
\), this implies
\[
  \E_{k-1}\left[\frac{g_{k,i}^{2}}
       {(\sqrt{v_{k-1,i}}+\varepsilon)^{2}}\right]
  \;\le\;
  \frac{G^{2}+\sigma^2}
       {\varepsilon^{2}}
  \quad
  \text{for all } i.
\]
Summing over \(i=1,\dots,d\) and factoring out \(L/2\) yields
\[
  \E_{k-1}[B_k]
  \;\le\;
  \frac{L}{2}\,d\,\frac{(G^{2}+\sigma^2)}{\varepsilon^{2}}
  =\,
  \frac{L\,d\,(G^{2}+\sigma^2)}{2\,\varepsilon^{2}},
\]
By the total expectation law
\[
  \E[B_k]
  \;\le\;
  \frac{L\,d\,(G^{2}+\sigma^2)}{2\,\varepsilon^{2}}
\]
which completes the proof.
\end{proof}
\subsection{Main result}
\begin{theorem}[Convergence to a stationary point]\label{thm:main}
Let Assumption~\ref{ass:all} hold and set the learning rate as in~\eqref{eq:eta-schedule}. Then for any horizon $K \ge 1$,
\[
\min_{0 \le k < K} \E\bigl[\|\nabla f(x_k)\|^2_2\bigr]
\le
 \mathcal{O}\!\left(\frac{\log K}{K^{1/4}}\right).
\]
\end{theorem}

\begin{proof}
The proof starts from the descent guarantee provided by the $L$-smoothness of $f(x)$, as stated in Equation~\eqref{eq:smooth-descent}:
\[
f(x_{k+1}) \le f(x_k) - \eta_k A_k + \eta_k^2 B_k.
\]
Taking $\mathbb{E}_{k-1}$, we get:
\[
\E[f(x_{k+1})] \le \E[f(x_k)] - \eta_k \E[A_k] + \eta_k^2 \E[B_k].
\]

We now bound the terms $\E[A_k]$ and $\E[B_k]$ using the provided lemmas.

\begin{enumerate}
    \item \textbf{Bounding $\E[A_k]$}: From Lemma~\ref{lem:Ak-lower}, we have:
    \[
    \E[A_k] \ge \frac{1}{\sqrt{2}\,G+\varepsilon} \E[\|\nabla f(x_k)\|^2_2] - \frac{C_m C_v}{(k+1)^{\frac{1}{4}}}
    \]
    
    \item \textbf{Bounding $\E[B_k]$}: From Lemma~\ref{lem:Bk}, $B_k$ is uniformly bounded. Define $C_b := \frac{L\,d\,(G^{2}+\sigma^2)}{2\,\varepsilon^{2}}$, then:
    \[
    \E[B_k] \le C_b.
    \]
\end{enumerate}

Substituting into the main inequality:
\[
\E[f(x_{k+1})] \le \E[f(x_k)] - \frac{\eta_k}{\sqrt{2}\,G+\varepsilon} \E\bigl[\|\nabla f(x_k)\|^2_2\bigr] + \eta_k \frac{C_m C_v}{k^{1/4}} + \eta_k^2 C_b.
\]

Let $\eta_k = \frac{\eta}{(k+2)^{3/4}} \leq \frac{\eta}{(k+1)^{3/4}}$. Then:
\[
\E[f(x_{k+1})] \le \E[f(x_k)] - \frac{\eta}{(k+1)^{3/4}(\sqrt{2}\,G+\varepsilon)} \E\bigl[\|\nabla f(x_k)\|^2_2\bigr] +\frac{C_m C_v \eta}{k+1} + \frac{\eta^2 \; C_b}{(k+1)^{3/2}}
\]

Rewriting:
\[
\frac{\eta}{(k+1)^{3/4}(\sqrt{2}\,G+\varepsilon)} \E\bigl[\|\nabla f(x_k)\|^2_2\bigr] \le \E[f(x_k)] - \E[f(x_{k+1})] +\frac{C_m C_v \eta}{k+1} + \frac{\eta^2 \; C_b}{(k+1)^{3/2}}
\]

Summing from $k=0$ to $K-1$:
\[
\sum_{k=0}^{K-1}\frac{\eta}{(k+1)^{3/4}(\sqrt{2}\,G+\varepsilon)} \E\bigl[\|\nabla f(x_k)\|^2_2\bigr] \le f(x_0) - f^\star +C_m C_v \eta\sum_{k=0}^{K-1}\frac{1}{k+1} +\eta^2 \; C_b \sum_{k=0}^{K-1}\frac{1}{(k+1)^{3/2}}
\]

The harmonic sum satisfies:
\[
\sum_{k=0}^{K-1} \frac{1}{k+1} \le 1 + \log K.
\]

The second term is a convergent \(p\)-series with exponent \(p = 3/2 > 1\); hence it converges to a finite limit as \(K \to \infty\). Therefore, the partial sum can be bounded by the value of the full series:
\[
\eta^2\, C_b \sum_{k=0}^{K-1} (k+1)^{-3/2}
 \;\le\; \eta^2\, C_b \sum_{k=0}^{\infty} k^{-3/2}
 \;=\; \eta^2\, C_b \,\zeta\!\left(\tfrac{3}{2}\right)
 \;=: C_R,
\]
where \(\zeta(\cdot)\) denotes the Riemann zeta function. We thus define the constant \(C_R := \eta^2 C_b \zeta\!\left(\tfrac{3}{2}\right)\).

By combining these terms, we have for the main equation:
\[
\sum_{k=0}^{K-1}\frac{1}{(k+1)^{3/4}} \E\bigl[\|\nabla f(x_k)\|^2_2\bigr] \le \frac{1}{C_{LHS}}\left[f(x_0) - f^\star + C_{sum}(1 + \log K) + C_R \right]
\]
with $C_{sum} = C_m C_v \eta$ and $C_{LHS} = \frac{\eta}{\tilde{G}+\epsilon}$.

The left-hand side can be lower-bounded as follows:
\[
\sum_{k=1}^{K-1}\frac{1}{(k+1)^{3/4}} \E\bigl[\|\nabla f(x_k)\|^2_2\bigr] \ge \min_{0 \le k < K} \E\bigl[\|\nabla f(x_k)\|^2_2\bigr] \sum_{k=1}^{K-1}\frac{1}{(k+1)^{3/4}}
\]

Now, let's find a lower bound for the sum. We can approximate it with an integral:
\[
\sum_{k=0}^{K-1}\frac{1}{(k+1)^{3/4}} = \sum_{j=1}^{K}\frac{1}{j^{3/4}} \ge \int_{1}^{K+1} \frac{1}{x^{3/4}} dx = \left[4x^{1/4}\right]_{1}^{K+1} = 4\left((K+1)^{1/4} - 1\right)
\]
For a large $K$, this sum is of the order $\mathcal{O}(K^{1/4})$.

Substituting this back into our main inequality, we get:
\[
4\left((K+1)^{1/4} - 1\right) \min_{0 \le k < K} \E\bigl[\|\nabla f(x_k)\|^2_2\bigr] \le \frac{1}{C_{LHS}}\left[f(x_0) - f^\star + C_{sum}(1 + \log K) + C_R \right]
\]

Now, we can isolate the minimum of the expected squared norm of the gradient:
\[
\min_{0 \le k < K} \E\bigl[\|\nabla f(x_k)\|^2_2\bigr] \le \frac{f(x_0) - f^\star + C_{sum}(1 + \log K) + C_R}{4 C_{LHS} \left((K+1)^{1/4} - 1\right)}
\]

As $K \to \infty$, the dominant terms are $\log K$ in the numerator and $K^{1/4}$ in the denominator. Therefore, we can write the convergence rate as:
\[
\min_{0 \le k < K} \E\bigl[\|\nabla f(x_k)\|^2_2\bigr] = \mathcal{O}\left(\frac{\log K}{K^{1/4}}\right) = \mathcal{\tilde{O}}\left(\frac{1}{K^{1/4}}\right)
\]

\begin{remark}[On the choice of hyperparameters]
The convergence analysis required carefully tuned hyperparameters. In particular, we employed a progressively increasing momentum coefficient, $\beta_{1,k} = 1 - 1/\sqrt{k}$, along with a relatively aggressive learning rate schedule. This combination was essential to control key residual terms throughout the proof and to ensure overall convergence. While more conservative choices failed to yield meaningful bounds, the schedule used here proved sufficient for our theoretical guarantees.
\end{remark}

\begin{remark}[On the convergence rate]
The convergence rate $\mathcal{O}(\log K / K^{1/4})$ for $\min_k \E[\|\nabla f(x_k)\|^2]$ arises intrinsically from the sign-based nature of the update rule. Unlike gradient-magnitude-based methods, our approach relies solely on the sign of momentum terms, which demands high directional accuracy. Ensuring this accuracy requires a tight control over the momentum error variance, which in turn necessitates a fast-decaying learning rate, $\eta_k = \mathcal{O}(k^{-3/4})$. This schedule guarantees reliable update directions but slows the overall convergence, as the learning rate bounds the algorithm’s progress. Hence, the rate reflects a fundamental trade-off: stability of sign-based directions versus the speed of descent.
\end{remark}
\end{proof}

\section{Additional Results}
\label{app:additionnal Results}

\subsection{Full Results on Analysis Parts}
\begin{table}[h]
\centering
\begin{tabular}{lccccc}
\toprule
\textbf{Optimizer} & $\sigma\!=\!0$ & 0.01 & 0.05 & 0.10 & 0.20 \\
\midrule
Ano   & 82.10$\pm$0.20 & 78.71$\pm$0.20 & 70.88$\pm$0.34 & 65.93$\pm$0.33 & 59.54$\pm$0.66 \\
Adam  & 80.67$\pm$0.37 & 75.97$\pm$0.27 & 66.86$\pm$0.52 & 60.83$\pm$0.54 & 52.46$\pm$0.93 \\
Lion  & 81.04$\pm$0.30 & 77.80$\pm$0.19 & 69.62$\pm$0.28 & 64.02$\pm$0.77 & 56.82$\pm$0.78 \\
Grams & 71.34$\pm$0.33 & 77.90$\pm$0.02 & 70.57$\pm$0.23 & 65.47$\pm$0.32 & 58.80$\pm$0.57 \\
\bottomrule
\end{tabular}
\caption{CIFAR-10 test accuracy (\%) with 95\% confidence intervals.}
\label{tab:noise-robustness-detailed}
\end{table}

\begin{table}[h]
\centering
\resizebox{\textwidth}{!}{
\begin{tabular}{l|cc|cc|cc|cc|cc|c}
\toprule
\textbf{Optimizers} & \multicolumn{2}{c|}{HalfCheetah-v5} & \multicolumn{2}{c|}{Ant-v5} & \multicolumn{2}{c|}{Humanoid-v5} & \multicolumn{2}{c|}{Walker2d-v5} & \multicolumn{2}{c|}{Hopper-v5} & Avg. Norm \\
 & Score$\pm$IC & Norm & Score$\pm$IC & Norm & Score$\pm$IC & Norm & Score$\pm$IC & Norm & Score$\pm$IC & Norm & \\
\midrule
\textit{Default}\\
Adam        & $10549.48 \pm 721.55$ & 97.10 & $4336.64 \pm 698.72$ & 82.05 & $5357.14 \pm 211.97$ & 99.29 & $4462.51 \pm 588.77$ & 85.36 & $3164.71 \pm 600.48$ & 89.52 & 90.66 \\
RMSprop     & $10506.23 \pm 852.19$ & 96.71 & $4234.37 \pm 763.65$ & 80.11 & $5395.51 \pm 126.80$ & 100.00 & $4160.06 \pm 480.62$ & 79.57 & $2973.86 \pm 571.05$ & 84.12 & 88.10 \\
Adan        & $7805.20 \pm 1154.02$ & 71.84 & $2985.19 \pm 1018.79$ & 56.48 & $5080.74 \pm 305.26$ & 94.17 & $4092.13 \pm 379.92$ & 78.28 & $3222.62 \pm 235.25$ & 91.16 & 78.38 \\
Lion        & $9527.96 \pm 805.42$  & 87.70 & $4948.26 \pm 243.05$  & 93.62 & $98.22 \pm 32.33$    & 1.82  & $4612.63 \pm 367.77$ & 88.23 & $3087.27 \pm 628.06$ & 87.33 & 71.74 \\
Grams       & $6782.60 \pm 715.12$  & 62.43 & $3207.30 \pm 531.06$  & 60.68 & $5104.10 \pm 692.14$ & 94.60 & $3656.66 \pm 658.82$ & 69.95 & $1475.34 \pm 927.22$ & 41.73 & 65.88 \\
\textbf{Ano (Ours)}    & $10864.09 \pm 1052.24$ & 100.00 & $5285.44 \pm 729.86$ & 100.00 & $5255.62 \pm 815.92$ & 97.41 & $5227.86 \pm 436.49$ & 100.00 & $3535.32 \pm 780.96$ & 100.00 & 99.48 \\
\textbf{Anolog (Ours)} & $10557.05 \pm 560.70$  & 97.17 & $5089.12 \pm 522.94$  & 96.29 & $5242.78 \pm 173.98$ & 97.17 & $4606.02 \pm 478.36$ & 88.11 & $3314.12 \pm 539.95$ & 93.74 & 94.50 \\
\addlinespace
\textit{Tuned}\\
Adam        & $8243.01 \pm 2750.47$ & 69.70 & $5050.53 \pm 471.12$  & 90.82 & $5224.24 \pm 339.87$ & 100.00 & $4429.62 \pm 668.97$ & 90.83 & $2968.40 \pm 696.52$ & 82.62 & 86.79 \\
RMSprop     & $10096.62 \pm 2379.00$& 85.37 & $3509.99 \pm 827.41$  & 63.12 & $64.97 \pm 42.44$    & 1.24  & $4583.19 \pm 969.52$ & 93.98 & $2031.80 \pm 771.10$ & 56.55 & 60.05 \\
Adan        & $10822.40 \pm 475.75$ & 91.51 & $5239.69 \pm 270.96$  & 94.22 & $4792.62 \pm 904.44$ & 91.74 & $4686.83 \pm 502.28$ & 96.11 & $3514.42 \pm 143.57$ & 97.82 & 94.28 \\
Lion        & $10482.06 \pm 1018.86$& 88.63 & $4848.41 \pm 821.79$  & 87.18 & $1349.15 \pm 1322.56$& 25.82 & $4876.76 \pm 253.22$ & 100.00 & $3592.87 \pm 70.26$  & 100.00 & 80.33 \\
Grams       & $10533.70 \pm 866.69$ & 89.07 & $4607.59 \pm 505.08$  & 82.85 & $5147.04 \pm 487.55$ & 98.52 & $4644.45 \pm 498.08$ & 95.24 & $3147.82 \pm 605.03$ & 87.61 & 90.66 \\
\textbf{Ano (Ours)}    & $11826.22 \pm 700.46$  & 100.00 & $5561.17 \pm 400.26$  & 100.00 & $5158.30 \pm 313.97$ & 98.74 & $4804.34 \pm 359.02$ & 98.51 & $3226.05 \pm 504.58$ & 89.79 & 97.41 \\
\textbf{Anolog (Ours)} & $11198.28 \pm 771.94$  & 94.69 & $5095.64 \pm 722.28$  & 91.63 & $3137.03 \pm 1335.43$& 60.05 & $4563.86 \pm 834.12$ & 93.58 & $3321.91 \pm 472.93$ & 92.46 & 86.48 \\
\addlinespace
\textit{Best Version}\\
Adam        & $10549.48 \pm 721.55$ & 97.10 & $4336.64 \pm 698.72$ & 82.05 & $5357.14 \pm 211.97$ & 99.29 & $4462.51 \pm 588.77$ & 85.36 & $3164.71 \pm 600.48$ & 88.08 & 90.38 \\
RMSprop     & $10506.23 \pm 852.19$ & 96.71 & $4234.37 \pm 763.65$ & 80.11 & $5395.51 \pm 126.80$ & 100.00 & $4160.06 \pm 480.62$ & 79.57 & $2973.86 \pm 571.05$ & 82.77 & 87.83 \\
Adan        & $10822.40 \pm 475.75$ & 99.62 & $5239.69 \pm 270.96$ & 99.13 & $4792.62 \pm 904.44$ & 88.83 & $4686.83 \pm 502.28$ & 89.65 & $3514.42 \pm 143.57$ & 97.82 & 95.01 \\
Lion        & $10482.06 \pm 1018.86$& 96.48 & $4848.41 \pm 821.79$ & 91.73 & $1349.15 \pm 1322.56$& 25.01 & $4876.76 \pm 253.22$ & 93.28 & $3592.87 \pm 70.26$  & 100.00 & 81.30 \\
Grams       & $10533.70 \pm 866.69$ & 96.96 & $4607.59 \pm 505.08$ & 87.18 & $5147.04 \pm 487.55$ & 95.39 & $4644.45 \pm 498.08$ & 88.84 & $3147.82 \pm 605.03$ & 87.61 & 91.20 \\
\textbf{Ano (Ours)}    & $10864.09 \pm 1052.24$ & 100.00 & $5285.44 \pm 729.86$ & 100.00 & $5255.62 \pm 815.92$ & 97.41 & $5227.86 \pm 436.49$ & 97.41 & $3535.32 \pm 780.96$ & 98.40 & 99.16 \\
\textbf{Anolog (Ours)} & $10557.05 \pm 560.70$  & 97.17 & $5089.12 \pm 522.94$ & 96.29 & $5242.78 \pm 173.98$ & 97.17 & $4606.02 \pm 478.36$ & 97.17 & $3314.12 \pm 539.95$ & 92.24 & 94.20 \\
\bottomrule
\end{tabular}
}
\caption{Comparison of the average performance ($\pm$ CI95\%) and normalized scores of different optimizers across MuJoCo environments.}
\label{tab:drl-norm}
\end{table}

\subsection{Full Normalized Score for PPO}

\begin{table}[h]
\centering
\resizebox{\textwidth}{!}{
\begin{tabular}{l|cc|cc|cc|cc|cc|c}
\toprule
\textbf{Optimizers} & \multicolumn{2}{c|}{BattleZone-v5} & \multicolumn{2}{c|}{DoubleDunk-v5} & \multicolumn{2}{c|}{NameThisGame-v5} & \multicolumn{2}{c|}{Phoenix-v5} & \multicolumn{2}{c|}{Qbert-v5} & Avg. Norm \\
 & Score$\pm$IC & Norm & Score$\pm$IC & Norm & Score$\pm$IC & Norm & Score$\pm$IC & Norm & Score$\pm$IC & Norm & \\
\midrule
\textit{Default}\\
Adam   & $7615.00 \pm 1299.89$ & 94.07 & $-1.08 \pm 0.21$ & 97.62 & $665.35 \pm 64.89$  & 78.71 & $3443.30 \pm 240.68$ & 100    & $4257.80 \pm 135.90$ & 88.18 & 91.71 \\
RMSprop & $7686.67 \pm 859.43$  & 94.96 & $-0.67 \pm 0.22$ & 100    & $798.00 \pm 118.66$ & 94.40 & $3031.13 \pm 410.22$ & 88.03  & $4585.67 \pm 538.44$ & 94.96 & 94.47 \\
Adan   & $6480.00 \pm 702.14$   & 80.05 & $-0.91 \pm 0.23$ & 98.58  & $638.35 \pm 18.08$  & 75.51 & $2106.90 \pm 110.74$ & 61.19  & $2665.00 \pm 708.69$ & 55.19 & 74.11 \\
Lion   & $1392.00 \pm 139.66$   & 17.20 & $-1.46 \pm 0.65$ & 95.42  & $508.15 \pm 91.28$  & 60.11 & $2432.35 \pm 234.61$ & 70.64  & $3768.00 \pm 520.95$ & 78.03 & 64.28 \\
Grams  & $7908.00 \pm 687.61$   & 97.69 & $-0.93 \pm 0.24$ & 98.48  & $633.80 \pm 23.52$  & 74.97 & $2234.40 \pm 130.73$ & 64.89  & $3670.12 \pm 634.75$ & 76.00 & 82.41 \\
\textbf{Ano (Ours)}    & $8095.00 \pm 494.70$   & 100   & $-0.97 \pm 0.14$ & 98.26  & $845.35 \pm 56.39$  & 100   & $2813.00 \pm 243.04$ & 81.69  & $4828.80 \pm 386.73$ & 100   & 95.99 \\
\textbf{Anolog (Ours)} & $7485.00 \pm 1010.66$  & 92.46 & $-0.98 \pm 0.14$ & 98.19  & $751.05 \pm 74.60$  & 88.84 & $2983.00 \pm 236.42$ & 86.63  & $4773.75 \pm 602.39$ & 98.86 & 93.00 \\
\addlinespace
\textit{Tuned}\\
Adam   & $6430.00 \pm 864.51$   & 74.55 & $-0.98 \pm 0.20$ & 99.14  & $549.75 \pm 51.42$  & 66.39 & $406.90 \pm 105.28$ & 14.40  & $4486.62 \pm 683.06$ & 75.27 & 65.95 \\
RMSprop & $0.00 \pm 452.43$      & 0.00  & $-0.83 \pm 0.14$ & 100    & $47.50 \pm 96.74$   & 5.74  & $16.20 \pm 1.47$    & 0.57   & $72.50 \pm 54.93$   & 1.22  & 21.51 \\
Adan   & $4840.00 \pm 2601.35$  & 56.12 & $-0.95 \pm 0.23$ & 99.30  & $754.20 \pm 51.44$  & 91.08 & $2647.20 \pm 534.22$ & 93.71  & $4524.75 \pm 448.74$ & 75.91 & 83.22 \\
Lion   & $1324.00 \pm 218.94$   & 15.35 & $-2.38 \pm 1.03$ & 90.98  & $574.55 \pm 73.33$  & 69.38 & $2232.85 \pm 364.56$ & 79.04  & $3759.75 \pm 808.26$ & 63.07 & 63.57 \\
Grams  & $7715.00 \pm 627.92$   & 89.45 & $-1.35 \pm 0.75$ & 96.95  & $690.40 \pm 76.27$  & 83.37 & $1989.15 \pm 201.64$ & 70.42  & $5049.25 \pm 624.43$ & 84.71 & 84.98 \\
\textbf{Ano (Ours)}    & $8625.00 \pm 1870.44$  & 100   & $-0.91 \pm 0.21$ & 99.55  & $828.10 \pm 67.66$  & 100   & $2824.85 \pm 226.30$ & 100    & $5960.88 \pm 912.36$ & 100   & 99.91 \\
\textbf{Anolog (Ours)} & $1470.00 \pm 1176.65$  & 17.04 & $-1.06 \pm 1.06$ & 98.67  & $543.60 \pm 111.65$ & 65.64 & $832.50 \pm 125.46$ & 29.47  & $1323.50 \pm 1694.70$ & 22.20 & 46.61 \\
\addlinespace
\textit{Best Version}\\
Adam   & $7615.00 \pm 1299.89$  & 88.29 & $-1.08 \pm 0.21$ & 97.62  & $665.35 \pm 64.89$  & 80.35 & $3443.30 \pm 240.68$ & 100    & $4257.80 \pm 135.90$ & 71.43 & 87.54 \\
RMSprop & $7686.67 \pm 859.43$   & 89.12 & $-0.67 \pm 0.22$ & 100    & $798.00 \pm 118.66$ & 96.37 & $3031.13 \pm 410.22$ & 88.03  & $4585.67 \pm 538.44$ & 76.93 & 90.09 \\
Adan   & $4840.00 \pm 2601.35$  & 56.12 & $-0.95 \pm 0.23$ & 98.35  & $754.20 \pm 51.44$  & 91.08 & $2647.20 \pm 534.22$ & 76.88  & $4524.75 \pm 448.74$ & 75.91 & 79.67 \\
Lion   & $1392.00 \pm 139.66$   & 16.14 & $-1.46 \pm 0.65$ & 95.42  & $508.15 \pm 91.28$  & 61.36 & $2432.35 \pm 234.61$ & 70.64  & $3768.00 \pm 520.95$ & 63.21 & 61.36 \\
Grams  & $7715.00 \pm 627.92$   & 89.45 & $-1.35 \pm 0.75$ & 96.03  & $690.40 \pm 76.27$  & 83.37 & $1989.15 \pm 201.64$ & 57.77  & $5049.25 \pm 624.43$ & 84.71 & 82.26 \\
\textbf{Ano (Ours)}    & $8625.00 \pm 1870.44$  & 100   & $-0.91 \pm 0.21$ & 98.60  & $828.10 \pm 67.66$  & 100   & $2824.85 \pm 226.30$ & 82.04  & $5960.88 \pm 912.36$ & 100   & 96.13 \\
\textbf{Anolog (Ours)} & $7485.00 \pm 1010.66$  & 86.78 & $-0.98 \pm 0.14$ & 98.19  & $751.05 \pm 74.60$  & 90.70 & $2983.00 \pm 236.42$ & 86.63  & $4773.75 \pm 602.39$ & 80.08 & 88.48 \\
\bottomrule
\end{tabular}
}
\caption{Comparison of the average performance ($\pm$ IC95\%) and normalized scores of different optimizers across Atari environments.}
\label{tab:atari-norm}
\end{table}

\section{Hyperparameters Settings}

\subsection{Optimizers Hyperparameters}
\label{app:all_hyperparameters_settings}

\begin{center}
\begin{longtable}{@{}lcccccc@{}}
\caption{Optimizers hyperparameter settings used in all experiments.}
\label{tab:optimizers-hyperparameters}\\
\toprule
\textbf{Model} & \textbf{Optimizer} & $\beta_1$ & $\beta_2$ & $\beta_3$ & $lr$ & $\lambda$ \\
\midrule
\endfirsthead

\multicolumn{7}{c}{{\bfseries Table \thetable\ (continued)}}\\
\toprule
\textbf{Model} & \textbf{Optimizer} & $\beta_1$ & $\beta_2$ & $\beta_3$ & $lr$ & $\lambda$ \\
\midrule
\endhead

\midrule
\multicolumn{7}{r}{}\\
\endfoot

\bottomrule
\endlastfoot

\multicolumn{7}{c}{\textit{Noise Robustness Analysis}}\\
\midrule
CNN & AdamW & 0.9   & 0.999 & -- & 1e-3  & -- \\
    & Lion  & 0.9   & 0.99  & -- & 1e-4  & -- \\
    & Ano   & 0.92  & 0.99  & -- & 1e-4  & -- \\
    & Grams & 0.9   & 0.999 & -- & 1e-3  & -- \\

\midrule
\multicolumn{7}{c}{\textit{Computer Vision (CIFAR-100)}}\\
\midrule
\multicolumn{7}{l}{\textit{Baseline}}\\
ResNet-34 & AdamW & 0.9  & 0.999 & --   & 1e-3 & 1e-2 \\
          & Adan  & 0.98 & 0.92  & 0.99 & 1e-3 & 1e-2 \\
          & Ano   & 0.92 & 0.99  & --   & 1e-3 & 1e-2 \\
          & Lion  & 0.9  & 0.99  & --   & 1e-3 & 1e-2 \\
          & Grams & 0.9  & 0.999 & --   & 1e-3 & 1e-2 \\
\multicolumn{7}{l}{\textit{Tuned}}\\
ResNet-34 & AdamW & 0.9  & 0.99  & --   & 1e-3 & 1e-2 \\
          & Adan  & 0.95 & 0.92  & 0.9 & 1e-3 & 1e-2 \\
          & Ano   & 0.95 & 0.95 & --   & 1e-4 & 1e-2 \\
          & Lion  & 0.92  & 0.99  & --   & 1e-4 & 1e-2 \\
          & Grams & 0.9  & 0.99 & --   & 1e-3 & 1e-2 \\

\midrule
\multicolumn{7}{c}{\textit{Natural Language Processing (GLUE)}}\\
\midrule
\multicolumn{7}{l}{\textit{Baseline}}\\
BERT (base) & AdamW & 0.9  & 0.999 & --   & 2e-5 & 1e-2 \\
            & Adan  & 0.98 & 0.92  & 0.99 & 2e-5 & 1e-2 \\
            & Ano   & 0.92 & 0.99  & --   & 2e-5 & 1e-2 \\
            & Lion  & 0.9  & 0.99  & --   & 2e-5 & 1e-2 \\
            & Grams & 0.9  & 0.999 & --   & 2e-5 & 1e-2 \\
            & Anolog& --   & 0.999 & --   & 2e-5 & 1e-2 \\
\multicolumn{7}{l}{\textit{Tuned}}\\
BERT (base) & AdamW & 0.95  & 0.99  & --   & 7e-5 & 1e-2 \\
            & Adan  & 0.95 & 0.92  & 0.96 & 7e-5 & 1e-2 \\
            & Ano   & 0.9 & 0.95 & --   & 2e-5 & 1e-2 \\
            & Lion  & 0.9  & 0.9  & --   & 7e-6 & 1e-2 \\
            & Grams & 0.92  & 0.999 & --   & 7e-5 & 1e-2 \\
            & Anolog& --   & 0.95 & --   & 2e-5 & 1e-2 \\

\midrule
\multicolumn{7}{c}{\textit{Deep Reinforcement Learning (MuJoCo \& Atari)}}\\
\midrule
\multicolumn{7}{l}{\textit{Baseline}}\\
SAC/PPO & Adam  & 0.9  & 0.999 & --   & 3e-4 & -- \\
    & Adan  & 0.98 & 0.92  & 0.99 & 3e-4 & -- \\
    & Ano   & 0.92 & 0.99  & --   & 3e-4 & -- \\
    & Lion  & 0.9  & 0.99  & --   & 3e-4 & -- \\
    & Grams & 0.9  & 0.999 & --   & 3e-4 & -- \\
    & RMSprop & -- & 0.99 & -- & 3e-4 & -- \\
    & Anolog& --   & 0.999 & --   & 3e-4 & -- \\
\multicolumn{7}{l}{\textit{Tuned}}\\
SAC/PPO & Adam  & 0.9 & 0.99 & --   & 3e-3 & -- \\
    & Adan  & 0.97 & 0.92  & 0.96 & 3e-3 & -- \\
    & Ano   & 0.95 & 0.99  & --   & 3e-4 & -- \\
    & Lion  & 0.92 & 0.99  & --   & 3e-4 & -- \\
    & Grams & 0.9  & 0.999  & --   & 3e-3 & -- \\
    & RMSprop & -- & 0.999 & -- & 3e-3 & -- \\
    & Anolog& --   & 0.99 & --   & 3e-3 & -- \\
\multicolumn{7}{l}{\textit{Best Version (SAC)}}\\
SAC & Adam  & 0.9 & 0.999 & --   & 3e-4 & -- \\
    & Adan  & 0.97 & 0.92  & 0.96 & 3e-3 & -- \\
    & Ano   & 0.92 & 0.99  & --   & 3e-4 & -- \\
    & Lion  & 0.92 & 0.99  & --   & 3e-4 & -- \\
    & Grams & 0.9  & 0.999  & --   & 3e-3 & -- \\
    & RMSprop & -- & 0.99 & -- & 3e-4 & -- \\
    & Anolog& --   & 0.999 & --   & 3e-4 & -- \\
\multicolumn{7}{l}{\textit{Best Version (PPO)}}\\
PPO & Adam  & 0.9 & 0.999 & --   & 3e-4 & -- \\
    & Adan  & 0.97 & 0.92  & 0.96 & 3e-3 & -- \\
    & Ano   & 0.95 & 0.99  & --   & 3e-4 & -- \\
    & Lion  & 0.92 & 0.99  & --   & 3e-4 & -- \\
    & Grams & 0.9  & 0.999  & --   & 3e-3 & -- \\
    & RMSprop & -- & 0.99 & -- & 3e-4 & -- \\
    & Anolog& --   & 0.999 & --   & 3e-4 & -- \\
\end{longtable}
\end{center}

\subsection{SAC Settings}

\begin{table}[H]
\centering
\begin{tabular}{l l}
\toprule
\textbf{Hyperparameter} & \textbf{Value} \\
\midrule
Total training steps & $1{,}000{,}000$ \\
Discount $\gamma$ & $0.99$ \\
Soft update rate $\tau$ & $0.005$ \\
Replay buffer size & $10^6$ \\
Batch size & $256$ \\
Learning starts & $5{,}000$ steps \\
Actor LR / Critic LR & see Tab \ref{tab:optimizers-hyperparameters} \\
Policy update frequency & $2$ \\
Target network update freq. & $1$ \\
Entropy coeff. $\alpha$ (init) & $0.2$ \\
Entropy autotune & \checkmark \\
Max grad. norm (actor) & $0.5$ \\
Logging interval & $2048$ steps \\
\bottomrule
\end{tabular}
\caption{SAC hyperparameters used in our MuJoCo experiments (values taken from the code).}
\label{tab:sac_hparams_main}
\end{table}

\subsection{PPO Settings}

\begin{table}[H]
\centering
\begin{tabular}{l l}
\toprule
\textbf{Hyperparameter} & \textbf{Value} \\
\midrule
Total timesteps & $10{,}000{,}000$ \\
Number of envs ($N_\mathrm{env}$) & $64$ \\
Steps per rollout ($N_\mathrm{steps}$) & $64$ \\
Batch size ($N_\mathrm{env}\!\times\!N_\mathrm{steps}$) & $4096$ \\
\;Minibatches & $4$ \\
\;Minibatch size & $1024$ \\
Update epochs & $4$ \\
Discount $\gamma$ & $0.99$ \\
GAE $\lambda$ & $0.95$ \\
Learning rate & see Tab \ref{tab:optimizers-hyperparameters}\\
LR annealing & linear (enabled) \\
Advantage normalization & \checkmark \\
Policy clip coef. & $0.10$ \\
Value clip & \checkmark \\
Entropy coef. & $0.01$ \\
Value loss coef. & $0.5$ \\
Max grad-norm & $0.5$ \\
Target KL & none \\
\bottomrule
\end{tabular}
\caption{PPO hyperparameters used in our Atari experiments (defaults from code).}
\label{tab:ppo_hparams_main}
\end{table}

\end{document}